\documentclass[twoside]{article}

\usepackage[accepted]{aistats2022arxiv}
\usepackage{amsmath,amsfonts}
\usepackage{amsthm}
\usepackage{amssymb}
\usepackage{bm}
\usepackage{bbm}
\usepackage{natbib}

\usepackage{color} 
\usepackage{algorithm}
\usepackage{algpseudocode}
\usepackage{subcaption}
\usepackage[hidelinks]{hyperref}
\usepackage{accents} % triangle as accent
\usepackage{relsize} % control size math symbols (triangle)

\DeclareBoldMathCommand{\p}{p}
\DeclareBoldMathCommand{\z}{z}
\DeclareBoldMathCommand{\v}{v}
\DeclareBoldMathCommand{\w}{w}
\DeclareBoldMathCommand{\y}{y}
\DeclareBoldMathCommand{\h}{h}
\DeclareBoldMathCommand{\q}{q}
\DeclareBoldMathCommand{\e}{e}
\DeclareBoldMathCommand{\b}{b}
\DeclareBoldMathCommand{\u}{u}
\DeclareBoldMathCommand{\L}{L}
\DeclareBoldMathCommand{\1}{1}
\DeclareBoldMathCommand{\0}{0}
\DeclareBoldMathCommand{\a}{a}
\DeclareBoldMathCommand{\x}{x}
\DeclareBoldMathCommand{\I}{I}
\DeclareBoldMathCommand{\bell}{\ell}

\renewcommand{\hat}{\widehat}
\renewcommand{\tilde}{\widetilde}
\newcommand{\regret}{\mathcal{R}}
\newcommand{\KL}{\textnormal{KL}}
\newcommand\Eb[1]{\E\left[ #1\right]}
\newcommand{\reals}{\mathbb{R}}
\newcommand{\lhat}{\hat{\ell}}
\newcommand{\Lhat}{\hat{L}}
\newcommand{\bLhat}{\hat{\L}}

\newcommand{\blhat}{\hat{\bell}}

\newcommand{\bltil}{\tilde{\bell}}
\newcommand{\ltil}{\tilde{\ell}}
\newcommand{\half}{\tfrac{1}{2}}
\newcommand{\id}{\mathbbm{1}}
\newcommand{\Phist}{F^{\star}}
\newcommand{\Phisttr}{\accentset{\mathsmaller{\mathsmaller{\mathsmaller{\mathsmaller{\triangle}}}}}{F}^{\star}}
\newcommand{\dtmax}{\rho_t^{\max}}
\newcommand{\dTmax}{\rho_T^{\max}}
\newcommand{\dstar}{\rho^{\star}}
\newcommand{\Otil}{\widetilde{O}}

\newcommand{\sumK}{\sum_{i=1}^K}

\newcommand{\sumKip}{\sum_{i'=1}^{K'}}
\newcommand{\sumT}{\sum_{t=1}^T}
\newcommand{\sumTnolim}{\sum\nolimits_{t=1}^T}

\newcommand{\sumt}{\sum_{s=1}^t}
\newcommand{\sumgam}{\sum_{\gamma \in \Gamma}}

\newcommand\inner[2]{\langle #1, #2 \rangle}
\newcommand{\diag}{\textnormal{diag}}
\newcommand{\Sset}{\mathcal{S}}
\newcommand{\Bset}{\mathcal{B}}

\DeclareMathOperator{\E}{\mathbb E}

\DeclareMathOperator*{\argmin}{arg\,min}

\newcommand{\todo}[1]{%
\ifmmode
\text{\textcolor{red}{TODO: #1}}
\else
\textcolor{red}{TODO: #1}
\fi
}

\usepackage{thmtools, thm-restate}

\newtheorem{theorem}{Theorem}
\newtheorem{lemma}{Lemma}

\begin{document}



\twocolumn[

% \aistatstitle{Arm-Dependent Delays with Full Information or Bandit Feedback}
% \aistatstitle{Non-Uniform Delays with Full Information or Bandit Feedback}
% \aistatstitle{Arm-Dependent Delays for Non-Stochastic Bandits and Experts}
% \aistatstitle{First-Order Regret Bounds and Arm-Dependent Delays for Non-Stochastic Bandits and Experts}
% \aistatstitle{Arm-Dependent Delays and First-Order Regret Bounds for Non-Stochastic Bandits and Experts}
% \aistatstitle{Non-Uniform Delays and First-Order Regret Bounds for Non-Stochastic Bandits and Experts}
% \aistatstitle{Non-Stochastic Bandits and Experts with Arm-Dependent Delays and First-Order Regret Bounds}
\aistatstitle{Nonstochastic Bandits and Experts with Arm-Dependent Delays}
\runningtitle{Arm-Dependent Delays}
\aistatsauthor{%
  Dirk van der Hoeven \\
  dirk@dirkvanderhoeven.com \\
  Dept.\ of Computer Science \\
  Università degli Studi di Milano, Italy \\
  \And
  Nicolò Cesa-Bianchi \\
  nicolo.cesa-bianchi@unimi.it \\
  DSRC \& Dept.\ of Computer Science \\
  Università degli Studi di Milano, Italy\\
} 
\aistatsaddress{~}]

\begin{abstract}
We study nonstochastic bandits and experts in a delayed setting where delays depend on both time and arms. While the setting in which delays only depend on time has been extensively studied, the arm-dependent delay setting better captures real-world applications at the cost of introducing new technical challenges.
In the full information (experts) setting, we design an algorithm with a first-order regret bound that reveals an interesting trade-off between delays and losses. We prove a similar first-order regret bound also for the bandit setting, when the learner is allowed to observe how many losses are missing.
These are the first bounds in the delayed setting that depend on the losses and delays of the best arm only.
When in the bandit setting no information other than the losses is observed, we still manage to prove a regret bound through a modification to the algorithm of \citet{zimmert2020optimal}.
Our analyses hinge on a novel bound on the drift, measuring how much better an algorithm can perform when given a look-ahead of one round.
\end{abstract}

\section{INTRODUCTION} % first level headings are all caps

Delayed feedback is a common challenge in many online learning problems. For example, suppose you want to sell several shares of a company on the stock market, but you do not know yet at what price you want to sell them. To learn a good price, you could do a piecemeal sale and track the price at which shares sell.
However, shares on sale at a low price will most likely be bought quickly, while shares on sale at a high price could be bought only much later in time. Hence, depending on the price you set, a sale may be completed sooner or later. 
Similarly, when deciding which advertisement to show on a web page, one might receive feedback at different points in time depending on which advertisement you show. If you show an ad of an expensive car, it might take some time for the person who saw the advertisement to make the decision to buy the car
%: this person might have to discuss with their spouse.
Vice versa, showing an advertisement of something significantly cheaper than a car might not induce a long delay between impression and conversion. 

In this paper we work in a general setting of sequential decision-making that captures these problems. We consider two forms of feedback: full information feedback, in which after each prediction the learner observes the losses of all actions, and bandit feedback, in which the learner only observes the loss of the chosen action.
For both full information and bandit setting we assume an oblivious adversary, which is to say that both losses and delays are adversarially generated before the start of the first round.

Formally, the full information setting works as follows. In each round $t = 1, \ldots, T$ the learner issues a prediction $\q_t \in \triangle$ and suffers loss $\inner{\q_t}{\bell_t}$, where $\triangle = \big\{\q: q(i) \geq 0, \sumK q(i) = 1\big\}$ is the $K$-dimensional simplex. The loss $\ell_t(i) \in [0, 1]$ of arm $i$ in round $t$ is revealed to the learner at the end of round $t + d_t(i)$, where $d_t(i) \ge 0$.
% after issuing his prediction in that round.
In the full information setting, the goal is to have small regret $\regret_T(\u)$ against any $\u \in \triangle$ after any number $T$ of rounds, where
\begin{equation*}
    \regret_T(\u) = \sum\nolimits_{t=1}^T \inner{\q_t - \u}{\bell_t}
\end{equation*}
In the bandit setting, the learner issues prediction $i_t \in [K]$ and suffers loss $\ell_t(i_t)$. We consider two variants of the bandit setting. In the first variant, the learner receives the loss $\ell_t(i_t)$ as well as the number of missing observations $\rho_t(i_t) = |\{s: s < t,  s + d_s(i_t) \geq t\}|$ for arm $i_t$ at the end of round $t + d_t(i_t)$, when $\ell_t(i_t)$ is observed. We refer to this variant as the partially-concealed bandit setting. In the second variant, at the end of round $t + d_t(i_t)$ the learner just receives $\ell_t(i_t)$. We refer to the second variant as the concealed bandit setting. In both variants, the goal is to have small expected regret: $\Eb{\regret_T(\u)} = \Eb{\sum\nolimits_{t=1}^T (\ell_t(i_t) - \inner{\u}{\bell_t})}$,
%
% \begin{equation*}
%     \Eb{\regret_T(\u)} = \Eb{\sum\nolimits_{t=1}^T (\ell_t(i_t) - \inner{\u}{\bell_t})} % = \Eb{\sumT \inner{\q_t - \u}{\bell_t}},
% \end{equation*}
%
where the expectation is with respect to the randomness in the learner's actions. 

The reason why we differentiate between the partially-concealed and concealed bandit setting is the following. Due to arm-dependent delays, we face an additional challenge in the bandit setting. In the full information setting, $\rho_t(i)$ may be used to tune the algorithm, and is readily available to the learner who can keep track of the missing losses for each arm. In the bandit setting, however, the learner only knows whether or not a loss is missing for an arm if that arm was played. Therefore, the learner cannot compute $\rho_t(i)$ or have knowledge of $d_t(i)$ if arm $i$ was not played in round $t$. Assuming that the learner also observes $\rho_t(i_t)$ when $\ell_t(i_t)$ is observed is a relatively mild assumption slightly simplifying the aforementioned challenge. Other assumptions in the literature are stronger than ours. For example, a common assumption is that delays are the same in each round \citep{weinberger2002delayed, neu2010online, neu2014online, cesa2016delay} or known before issuing a prediction \citep{thune2019nonstochastic, bistritz2019exp3}\footnote{Although \citet{bistritz2019exp3} also claim a result that does not use a-priori knowledge of delay, there appears to be an error in their analysis: see the discussion by \citet{gyorgy2020adapting}.}. When all arms suffer the same delay, $\rho_t(i_t)$ can be readily computed by simply counting the number of missing observations at prediction time. Regardless, we also study the case in which the learner does not observe $\rho_t(i_t)$, and derive algorithms for both the partially concealed and the concealed bandit settings.

\paragraph{Related Works}

%We split the related works depending on which setting they consider. %We start with the full information setting. 

%\subsubsection{Full Information Setting}

In the full information setting, \citet{weinberger2002delayed} were the first to study online learning with delayed feedback. They assumed full information feedback and a constant and known delay. Subsequent works by \citet{zinkevich2009slow, joulani2013online, quanrud2015online} relaxed the assumption on the delay, and worked in the more general setting of online convex optimization. %a more general setting where predictions belong to an arbitrary convex domain. 
\citet{joulani2016delay} also work with a convex domain and provide a reduction for deriving algorithms with a gradient-dependent regret bound using a strongly convex regularizer. %Hence, using the reduction of \citet{joulani2016delay} and an entropic regularizer, one obtains a regret bound that scales with the sum of the maximums of the losses in each round. 
Recently, \citet{flaspohler2021online} analysed the role of optimism in online convex optimization with delays, and developed several adaptive algorithms for the arm-independent setting. %, which unfortunately leads to suboptimal bounds in the arm-dependent delay setting due to the analysis being based on the strong convexity of the regularizer. 

Other works consider the distributed or parallel online convex optimization setting, see \citep{ agarwal2012distributed, mcmahan2014delay, sra2015adadelay, hsieh2020multi, vanderhoeven2021distributed} and the references therein. In these works, delays arise due to communication between agents, gradient evaluations, or distance between agents in a network of agents. 

%\subsubsection{Nonstochastic Bandit Setting}

All previous works in the nonstochastic bandit setting considered delayed feedback assuming equal delays for each arm. The impact of delay under bandit feedback was first studied by \citet{neu2010online, neu2014online}, who assumed that the delay was constant and known. \citet{cesa2019delaycoop} stated a lower bound for nonstochastic bandits of order $\max\{\sqrt{KT}, \sqrt{\ln(K)D}\}$, where $D$ is the total delay. \citet{thune2019nonstochastic, bistritz2019exp3, bistritz2021nodiscounted} all provided algorithms based on EXP3 that matched the lower bound up to a factor $\ln(K)$ by utilizing either a priori knowledge of $D$ or knowledge of $d_t$ before issuing a prediction. \citet{zimmert2020optimal} provide an algorithm matching the lower bound without using any additional assumptions on $d_t$. \citet{gyorgy2020adapting} present an improved analysis of a variant of EXP3 under delayed feedback, and provide a delay adaptive regret bound, a data and delay adaptive regret bound, and a high-probability regret bound.

Another related setting is the nonstochastic bandits with composite anonymous feedback setting \citep{cesa2018nonstochastic}. In this setting, the learner does not observe the identity of the losses, which is to say that the learner does not see which round a given loss is from. To complicate matters further, losses from several rounds may be composed without the possibility to distinguish which rounds or actions are in the composed losses. As long as a uniform upper bound on the delays $d^{\star}$ is known, the algorithm by  \citet{cesa2018nonstochastic} also works in our setting, but using this algorithm would lead to a suboptimal regret bound of order $\sqrt{d^{\star}KT\ln(K)}$. 

%\subsubsection{Stochastic Bandits Setting}

A stochastic version of the concealed bandit setting has been considered by \citet{gael2020stochastic, lancewicki2021stochastic}. They study i.i.d.\ rewards (rather than losses) and assume a slightly more general feedback model in which the learner does not always see the reward for actions. %his actions. 
Other works in the stochastic setting include \citep{chapelle2011empirical, dudik2011efficient, joulani2013online, desautels2014parallelizing, chapelle2014modeling, mandel2015queue, vernade2017stochastic, vernade2018contextual, pike2018bandits}, all with varying assumptions on the delay, but all with uniform delay over the arms.

%\subsection{Contributions}

\paragraph{Contributions}

In the full information setting, we provide an algorithm that satisfies, for any $\u \in \triangle$, a first-order bound of the form
\begin{equation}\label{eq:fullintrobound}
    \regret_T(\u) = \Otil\left(\sqrt{\big(\KL(\u, \pi) + \ln(T)\big)\inner{\u}{\L_T + \L_T^\rho}}\right) %\Otil\left(\dTmax + \sqrt{\ln(K)\inner{\u}{\L_T + \L_T^\rho}}\right),
\end{equation}
%
%where $\dTmax = \max_i \max_t \rho_t(i)$, 
where $\pi$ is a prior set by the learner, $L_T(i) = \sumT \ell_t(i)$, $L_T^\rho(i) = \sumT \ell_t(i)\rho_t(i)$, and $\Otil$ hides logarithmic factors. This bound takes advantage of scenarios when there is an arm with consistently small losses, or when the delays of a given arm tend to be small. In the worst case, bound~\eqref{eq:fullintrobound} is $\Otil\big(\sqrt{T(1 + \dTmax)\ln(K)}\big)$ for a uniform prior $\pi$. While several first-order regret bounds are known in the literature---see for example \citep{hutter2004prediction, kalai2005efficient, cesa2006prediction, cesa2007improved, vanerven2014follow, chen2021impossible}---to the best of our knowledge there are no first-order regret bounds for the delayed feedback setting. The regret bound in~\eqref{eq:fullintrobound} reveals an interesting trade-off. When competing with $\e_{i^\star} = \argmin_{\u \in \triangle} \inner{\u}{\L_T}$ we could suffer a large penalty for delay if $i^\star$ always suffers large delays (and thus incur large regret). On the other hand, if we try to balance delay and losses by competing with $\argmin_{\u \in \triangle} \{\inner{\u}{\L_T} + \sqrt{\KL(\u, \pi)\inner{\u}{\L_T + \L_T^\rho}}\}$, we could incur a smaller penalty for the delay at the cost of moving away from the minimizer of the loss. In our first motivating example, these scenarios correspond to, respectively, the optimal price of a share always incurring a small loss but a large delay, and a price at which a share always sells quickly while possibly incurring a larger loss. 

In the partially concealed bandit setting we also provide an algorithm with a trade-off between losses and delays: our algorithm guarantees that the regret $\Eb{\regret_T(\u)}$ is bounded, for any $\u \in \triangle$, by %we provide an algorithm whose regret $\Eb{\regret_T(\u)}$ is bounded, for any $\u \in \triangle$, by
\begin{equation*}\label{eq:banditintrobound1}
    \Otil \Big(\sqrt{K\inner{\u}{\L_T}} + \sqrt{\big(\KL(\u, \pi) + \ln(T)\big)\inner{\u}{\L_T^\rho}}\Big)
\end{equation*}
%
%As in the full information setting there is a trade-off between competing with the $\u$ achieving the smallest loss and competing with a $\u$ suffering small delays. %Note also that, as in the full information setting, the cost of delay scales with $\KL(\u, \pi)$. This means that whenever arms that have a large delay are known in advance, the learner can account for that by adjusting the prior and assigning a small weight to these arms.
%
In the bandit literature there are several algorithms known to achieve first-order regret bounds. For example, the Green algorithm \citep{allenberg2006hannan}, FPL-TrIX \citep{neu2015first}, or FTRL with the log-barrier \citep{foster2016learning} all guarantee a first-order regret bound. In the delayed bandit literature, \citet{gyorgy2020adapting} provide an algorithm with a regret bound of order $\sqrt{d^\star \inner{\u}{\L_T}} + \sqrt{\ln(K)\sumK L_T(i)}$ for $d^\star = \max_{i,t} d_t(i)$, which---to the best of our knowledge---is the one closest to a first-order regret bound. While in the worst case this bound is a factor $\ln(T)$ better than % \eqref{eq:banditintrobound1} 
our bound, unlike our bound it does not capture the trade-off between delays and losses.  

In the concealed bandit setting, we circumvent the need of knowing $\rho_t(i_t)$ by modifying the algorithm of \citet{zimmert2020optimal}. Our variant guarantees an expected regret bound of order
\begin{equation*}\label{eq:banditintrobound2}
\begin{split}
    \sqrt{KT} + \sqrt{\ln (K)\Eb{\sumT \sumK q_t(i)\rho_t(i)}} + \dstar
\end{split}
\end{equation*}
where $\dstar \geq \max_{i, t} \rho_t(i)$ is the maximum number of missing observations and , we recall it, $\q_1,\ldots,\q_T$ are the learner's predictions. While our bound is stated in terms of $\sumK q_t(i)\rho_t(i)$, a hardly interpretable quantity which depends on the actions of the algorithm, we may also recover a bound in terms of $\sumK q_t(i)\rho_t(i) \leq \max_i \rho_t(i)$, the maximum number of missing observations per round. Note that $\sumK q_t(i)\rho_t(i)$ or even $\max_{i} \rho_t(i)$ are not known in the concealed bandit setting, making adaptive tuning necessary.  \citet{zimmert2020optimal} consider a special case of our setting in which 
% the delays for the arms are equal,
delays are equal across arms,
which we call the arm-independent delay setting. In the arm-independent delay setting $\sumK q_t(i)\rho_t(i) = \rho_t$ is known and may be used for tuning, which \citet{zimmert2020optimal} do to recover the optimal regret bound. However, in our arm-dependent delay setting we do not know $\sumK q_t(i)\rho_t(i)$. We address this challenge by modifying the algorithm of \citet{zimmert2020optimal}, with which we recover the results of \citet{zimmert2020optimal} up to the additive $\dstar$ in the arm-independent setting as well as providing the bound in \eqref{eq:banditintrobound2} in the arm-dependent delay setting. The data-adaptive algorithm of \citet{gyorgy2020adapting} could also be applied in the concealed bandit setting, as they refrain from using any information not available to the learner in the tuning of their algorithm. However, this would yield an additional factor $\ln(K)$ and an additive $(\dstar)^2$ in the regret bound, rather than the additive $\dstar$ term which we obtain.

As for the technical contributions, our algorithm is an instance of Follow the Regularized Leader (FTRL) with corrections. \citet{hazan2010extracting, steinhardt2014adaptivity, wei2018more, chen2021impossible} also use FTRL with corrections, but solely in the non-delayed setting. Our contribution here is the analysis of the FTRL with corrections algorithm under delayed feedback. In particular, we provide a powerful tool for deriving regret bounds in both the full information and bandit settings. 
In our analysis we build upon the framework of \citet{gyorgy2020adapting}, who split the regret into a cheating regret term and a drift term. The cheating regret, which is the regret of an omniscient algorithm knowing all past losses (including the loss of the current round), can be bounded by standard techniques. The drift is the difference between % the regret of the omniscient predictor and the regret of the algorithm.
the loss of the algorithm and the loss of the omniscient predictor.
While \citet{gyorgy2020adapting} only provide bounds on the drift term for a version of exponential weights, we provide a general bound for FTRL algorithms under mild conditions on the regularizer, thus significantly simplifying the subsequent analysis. In the full information setting, we use this new bound to simultaneously learn the optimal learning rate and the best arm using a particular version of the negative entropy as regularizer. In the partially concealed bandit setting, we develop a new hybrid regularizer that allows us to simultaneously adapt to delay and loss of the best arm. In the concealed bandit setting, we show the power of our new bound by recovering the nontuned results of \citet{zimmert2020optimal}, and subsequently adapting the tuning to suit our purposes. 

The remainder of the paper is organised as follows. In Section~\ref{sec:algorithm} we introduce the algorithm and the corresponding analysis, which we use to derive our results. In Section~\ref{sec:Full information} we specify the regularizer used in the aforementioned algorithm and use the analysis from Section~\ref{sec:algorithm} to prove a regret bound. In Section~\ref{sec:bandits} we describe and analyze the regularizers used to derive the results for the concealed and partially concealed bandit settings. Finally, in section~\ref{sec:conclusion} we discuss some open problems. 

\section{ALGORITHM}\label{sec:algorithm}
One of the main challenges in designing an adaptive algorithm that can handle delayed feedback is that tuning the learning rate may involve unknown losses due to the delay. In order to resolve this issue, we learn the best learning rate by specifying multiple learning rates for each expert. This means that we create several pseudo-experts and the algorithm computes distribution a $\p_t^{av}$ over the pseudo-experts experts in each round. To compute the actual predictions we sum the weights over pseudo-experts that share an expert, i.e., we sum over the learning rates. A precise definition of our predictions can be found in the subsequent sections.

\begin{algorithm}[t]
\caption{Delayed FTRL with Corrections}\label{alg:framework}
\begin{algorithmic}[1]
\Require Regularizers $\Phi_t$ and $\Psi_t$, priors $\p_1^\Phi$ and $\p_1^\Psi$. 
\For{$t = 1 \ldots T$}
\State Using $R_t$ defined in \eqref{eq:FTRLregularizer}, compute % and return
\begin{equation}\label{eq:frameworkprediction}
    \p_t^{av} = \argmin_{\p \in \triangle'} \inner{\p}{\bLhat_t^{av}} + R_t(\p)
\end{equation}
\For{$i' \in K'$}
%\State \textbf{for}  $i' \in K'$ \textbf{do}
\For{$s \in A_t(i') \equiv \{s: s + d_s(i') = t\}$}
\State Get loss $\ltil_s(i')$ and correction $a_s(i')$
\State Set $\lhat_s(i') = \ltil_s(i') + a_s(i')$  
\EndFor
\State Update $\Lhat_{t+1}^{av}(i') = \Lhat_t^{av}(i') + \sum_{s \in A_t(i')} \lhat_s(i')$
%\State \textbf{end for}
\EndFor
\EndFor
\end{algorithmic}
\end{algorithm}

Our version of FTRL with corrections is given in Algorithm \ref{alg:framework}. Note that both the correction $a_t(i')$ and loss (estimate) $\ltil_t(i')$ are left unspecified. $\ltil_t(i')$ has different definitions in the bandit and the full information setting, as only in the bandit setting we need to estimate the loss.
% whereas in the full information setting the loss does not need to be estimated.
Similarly, $a_t(i')$ will be specified in the relevant sections as different choices for $a_t(i')$ lead to different regret bounds. We use this flexibility to its full extent by choosing different corrections for the full information, partially concealed bandit, and concealed bandit settings.

Our analysis builds on the following simple but useful decomposition due to \citet{gyorgy2020adapting}:
\begin{equation}\label{eq:driftcheating}
\begin{split}
    & \sum\nolimits_{t=1}^T \inner{\p_t^{av} - \u}{\blhat_t} \\
    & = \underset{\textnormal{Cheating regret}}{\underbrace{\sum\nolimits_{t=1}^T \inner{\p_{t+1} - \u}{\blhat_t}}} + \underset{\textnormal{Drift}}{\underbrace{\sum\nolimits_{t=1}^T \inner{\p_t^{av} - \p_{t+1}}{\blhat_t}}}
\end{split}
\end{equation}
Here $\p_{t+1}$ is the FTRL distribution when all losses up to and including the loss at round $t$ are known:
\begin{align}
\label{eq:lookahead}
    \p_{t+1} = \argmin\nolimits_{\p \in \triangle'} \inner{\p}{\bLhat_t} + R_t(\p)
\end{align}
We let $\bLhat_t = \sum_{s = 1}^t \blhat_s$ and also let $\triangle'$ be the enlarged simplex $\{\p : p(i') \geq 0, \sumKip p(i') = 1\}$ over $K' \geq K$ coordinates. %Note that the regularizer used to compute $\p_{t+1}$ is the same regularizer used to compute $\p_{t}^{av}$, which allows us to relate $\p_{t+1}$ to $\p_{t}^{av}$ for bounding the drift.
Note our notational convention here: $\bLhat_t^{av}$ contains all the information available to the learner \emph{before} issuing a prediction in round $t$ and is used to compute $\p_t^{av}$, whereas $\bLhat_t$ contains all the losses up to and \emph{including} $\blhat_t$. 

For some $\Psi_t$ and $\Phi_t$ to be specified later, we always use regularizers of the form 
\begin{equation}\label{eq:FTRLregularizer}
    R_t(\p) =  B_{\Psi_t}(\p, \p^\Psi_1) + B_{\Phi_t}(\p, \p^\Phi_1)
\end{equation}
where $B_F(\x, \y) = F(\x) - F(\y) - \inner{\nabla F(\y)}{\x - \y}$ is the Bregman divergence between $\x$ and $\y$ generated by $F$. We require that $F_t = \Psi_t + \Phi_t$ be a Legendre function---see, for example, \citep[Chapter 11.2]{cesa2006prediction} for a definition---with a positive definite Hessian on $\textnormal{dom}(F_t)$. Throughout the paper, $\textnormal{dom}(F_t)$ is the positive cone. The motivation for using a hybrid regularizer comes from \citet{zimmert2020optimal}, who show how to obtain optimal bounds in the nonstochastic bandit setting by using a regularizer to control the variance of the loss estimates and a different regularizer to control the impact of delay. 

As discussed by \citet{gyorgy2020adapting}, the challenge for most delayed (bandit) algorithms is to control the drift in terms of $\p_t^{av}$. The following Lemma, whose proof can be found in Appendix \ref{app:algorithm}, provides a general bound on the drift when the Hessian of $R_t$ at $\p_t^{av}$ satisfies some mild conditions. In the following, we use $\u \le \v$ to denote $u_i \le v_i$ for all $i$ and we use $\u < \v$ if $\u \le \v$ and $u_i < v_i$ for at least one $i$.
We define $G_t(\y,\z) = \inner{\nabla F_t^\star(\y)}{\z}$ and $W_t(\x, \z) = \inner{\nabla F_t(\x)}{\z}$ for any vectors $\x$ and $\y$. 
\begin{restatable}{relemma}{lemdrift}
\label{lem:drift}
For any $\z > \0$, if $W_t(\cdot,\z)$ is concave, $W_t(\x_1, \z) \leq W_t(\x_2, \z)$ iff $\x_1 \leq \x_2$, and $W_t(\x_1, \1) < W_t(\x_2, \1)$ iff $\x_1 < \x_2$, then
\begin{equation*}
    \inner{\p_t^{av} -\p_{t+1} }{\blhat_t} \leq (\bLhat_t - \bLhat_t^{av})^\top \big(\nabla^{2} R_t(\p_t^{av})\big)^{-1} \blhat_t
\end{equation*}
\end{restatable}
Denote by $\Sset_t(i') = \{s : s + d_s(i') < t\}$ the indices of available losses for arm $i'$ at the beginning of round $t$, and by $m_t(i') = \{s: s < t, s + d_s(i') \geq t\}$ the set of missing losses for arm $i'$ at the beginning of round $t$, i.e., the losses that in a non-delayed full information setting would have been available to the learner for prediction at round $t$ but are not available due to delay. Note that $\rho_t(i') = |m_t(i')|$. The quantity $\Lhat_t(i') - \Lhat_t^{av}(i')$ is a central quantity and is given by 
\begin{equation}
\label{eq:key}
\begin{split}
    \Lhat_t&(i') - \Lhat_t^{av}(i')  = \sum_{s \leq t} \lhat_s(i') - \sum_{s \in \Sset_t(i')} \lhat_s(i') \\
    & = \sum_{s \in [t] \setminus\Sset_t(i')} \lhat_s(i')  = \lhat_t(i') + \sum_{s \in m_t(i')}\lhat_s(i')
\end{split}
\end{equation}
The challenge and the main difference with respect to the standard delayed setting comes from the $m_t(i')$ term. While in the standard delayed setting the $m_t(i')$ term is the same for all arms, in our setting the set of missing losses can be different for each arm. This challenge is addressed in the next sections. %We now continue by bounding the cheating regret. 
To bound the cheating regret we will use the following standard lemma---see, for example, \citep{joulani2020modular}:
\begin{restatable}{relemma}{lemcheating}
\label{lem:cheating regret}
Suppose $R_t(\p) \geq R_{t-1}(\p)$ for all $t \ge 1$ and $\p\in\triangle'$, where $R_0 = R_1$. Let $\p_0 = \argmin_{\p \in \triangle'} R_1(\p)$. Then, for any $\u \in \triangle'$ and for $\p_{t+1}$ defined in~\eqref{eq:lookahead}, $\sum_{t=1}^{T} \inner{\p_{t+1} - \u}{\blhat_t}  \leq R_{T}(\u)$.
%
% \begin{align*}
%     & \sum_{t=1}^{T} \inner{\p_{t+1} - \u}{\blhat_t}  \leq R_{T}(\u)% - R_1(\p_0)
% \end{align*}
\end{restatable}
For completeness we provide the proof of Lemma~\ref{lem:cheating regret} in Appendix~\ref{app:algorithm}. Since $\blhat_t = \bltil_t  + \a_t$, to recover a regret bound with respect to $\bltil_t$ we simply subtract $\sumT \inner{\p_t^{av} - \u}{\a_t}$ from both sides of~\eqref{eq:driftcheating}: 
\begin{equation}
\label{eq:decomp}
\begin{split}
    \sumT & \inner{\p_t^{av} - \u}{\bltil_t}  =  \sumT \inner{\u}{\a_t} + \sumT \inner{\p_{t+1} - \u}{\blhat_t} \\
    & + \sumT \inner{\p_t^{av} - \p_{t+1}}{\blhat_t} - \sumT \inner{\p_t^{av}}{\a_t}
\end{split}
\end{equation}
The $\sumT \inner{\p_t^{av}}{\a_t}$ term is used to control the drift. The $\sumT \inner{\u}{\a_t}$ term is what allows us to derive bounds that depend on $\u$. 
Applying Lemmas~\ref{lem:drift} and~\ref{lem:cheating regret} to~\eqref{eq:decomp}, we can prove the following result.
\begin{lemma}\label{lem:frameworkregret}
Under the assumptions of Lemmas~\ref{lem:drift} and~\ref{lem:cheating regret}, Algorithm~\ref{alg:framework} satisfies:
\begin{align*}
    \sumT & \inner{\p_t^{av} - \u}{\bltil_t}
\le
    \sumT \inner{\u}{\a_t} + R_T(\u) % - R_1(\p_0)
\\& +
    \sumT \left((\bLhat_t - \bLhat_t^{av})^\top \big(\nabla^{2} R_t(\p_t^{av})\big)^{-1} \blhat_t - \inner{\p_t^{av}}{\a_t}\right)
\end{align*}
\end{lemma}
The remaining challenge is to choose $\Psi_t$, $\Phi_t$, and $a_t$.% to bound the regret.

\section{FULL INFORMATION SETTING}\label{sec:Full information}
In the full information setting we use $\Phi_t \equiv 0$ and 
\begin{equation}\label{eq:fullinforeg}
    \Psi_t(\p) = \sumK \sum_{\eta \in H_t} \frac{p(i, \eta)}{\eta}  \ln(p(i, \eta))
\end{equation}
where $H_t$ is a multiset of learning rates we define below. 
%
% $\Phi_t(\p) = 0$ and $\Psi_t(\p) = \sumKip\frac{1}{\eta_t(i') p(i') \ln(p(i'))}$. 
% This means that the regularizer we provide to Algorithm \ref{alg:framework} is
% %
% \begin{equation}\label{eq:fullinforeg}
%     R_t(\p) = \sumKip \frac{1}{\eta_t(i')} \big(p(i') \ln \left(\frac{p_(i')}{p_1^\Psi(i')}\right) - p(i') + p_1^\Psi(i')\big),
% \end{equation} 
% %
% where $\p_1^\Psi$ can be chosen arbitrarily. 
%
Similar regularizers have been used by \citet{bubeck2017online} and \citet{chen2021impossible} to obtain multiscale algorithms. In particular, the algorithm of \citet{chen2021impossible} is related to ours, as they also use corrections. However, we are the first ones to analyze this regularizer in a delayed feedback setting. 

As mentioned already, we simultaneously learn to compete with arbitrary $\u \in \triangle$ and learn the best learning rate. We do so by using several learning rates for each expert. Denote by $H_t$ the multiset of learning rates available in round $t$ and let $i' = (i,\eta)$. We have that $|H_t| = J$ for all $t$, where we set $J = \big\lceil \log_2(\sqrt{T}) \big\rceil$. Then
% With each pair $i \in [K]$ and $\eta \in H_t$ we identify a unique $i'$.
$\p_t^{av}$ is a distribution over $[K] \times [J]$.
% , $p_t^{av}(i') = p_t^{av}(i, \eta)$.
The probability the algorithm assigns to expert $i$ is the marginal:
\begin{equation}\label{eq:fullinfopred}
    q_t(i) = \sum\nolimits_{\eta \in H_t} p_t^{av}(i, \eta)
\end{equation}
The two missing inputs to Algorithm~\ref{alg:framework} are $\ltil_t(i')$ and $a_t(i')$. In the full information setting we do not need to estimate $\ell_t(i)$, and can simply set $\ltil_t(i') = \ltil_t(i, \eta) = \ell_t(i)$. As we mentioned above, the correction term is chosen to control the drift
\begin{equation}\label{eq:fullinfocorr}
a_t(i') = a_t(i, \eta) = 4 \eta \ell_t(i) (1 + \rho_t(i))
\end{equation}
The multiset $H_t$ of learning rates in round $t$ contains 
\begin{equation}\label{eq:fullinfoeta}
\begin{split}
    \min\bigg\{ \frac{1}{4(1 + \dtmax)}, \frac{\sqrt{\ln(K) + 2(\ln(T) + 1)}}{4\sqrt{1+\dtmax}2^{j}}\bigg\}
\end{split}
\end{equation}
for all $j = 1, \ldots, J$, where $\dtmax = \max_{i \leq K,s \leq t} \rho_s(i)$. Hence for each $j$ there is a (not necessarily distinct) $\eta \in H_t$, which we use to define the prior:
\begin{equation}\label{eq:fullinfoprior}
    p_1^\Psi(i') = p_1^\Psi(i, \eta) = \frac{2^{-2j}}{\sum_{j' = 1}^J 2^{-2j'}}\pi(i)
\end{equation}
Note that this is similar to the prior used by \citet{chen2021impossible}. Next, we state the main result of this section in asymptotic form (see Theorem~\ref{th:fullinfobound} in Appendix~\ref{app:Full information} for the same bound with explicit constants).
\begin{theorem}\label{th:informal fullinfobound}
If we run Algorithm~\ref{alg:framework} with regularizer $\Phi_t \equiv 0$, $\Psi_t$ defined by~\eqref{eq:fullinforeg}, corrections $a_t$ defined by~\eqref{eq:fullinfocorr}, and prior defined by~\eqref{eq:fullinfoprior}, then for any $\u \in \triangle$ the predictions $\q_t$ defined by~\eqref{eq:fullinfopred} using $\p_t^{av}$ computed by the algorithm satisfy
\begin{align*}
     & \sum\nolimits_{t=1}^T  \inner{\q_t - \u}{\bell_t}  \\
     & = O\left(\dTmax + \sqrt{\big(\KL(\u, \pi) + \ln(T)\big)\inner{\u}{\L_T + \L_T^\rho}}\right)
\end{align*}
\end{theorem}
\begin{proof}[Sketch proof of Theorem \ref{th:informal fullinfobound}]
Denote by $\eta^\star \in H_T$ the target learning rate and let $u'(i') = u'(i, \eta) = 0$ if $\eta \not = \eta^\star$ and $u'(i, \eta) = u(i)$ otherwise. First, observe that by definition of $\q_t$, $\u'$, and $\bltil$ we have that $\sumT \inner{\q_t - \u}{\bell_t} = \sumT \inner{\p_t^{av} - \u'}{\bltil_t} $, 
%
% \begin{equation*}
% \begin{split}
%     & \sumT \inner{\q_t - \u}{\bell_t} = \sumT \inner{\p_t^{av} - \u'}{\bltil_t} 
% \end{split}
% \end{equation*}
%
where the inner product on the left-hand side is on $\reals^{K}$ and the inner products on the right-hand side are on $\reals^{K\times J}$.
Since $R_t$ can be shown to verify the conditions of Lemma~\ref{lem:frameworkregret}, we obtain
\begin{equation*}
\begin{split}
    & \sumTnolim \inner{\q_t - \u}{\bell_t} % = \sumT \inner{\p_t^{av} - \u'}{\bltil_t} \\
    \leq R_T(\u')  + \sumTnolim \inner{\u'}{\a_t} \\
    & + \sumTnolim \Big((\bLhat_t - \bLhat_t^{av})^\top \big(\nabla^{2} R_t(\p_t^{av})\big)^{-1} \blhat_t - \inner{\p_t^{av}}{\a_t}\Big)
\end{split}
\end{equation*}
With some work we may then prove a bound on $R_T(u')$ of order $O\big((\KL(\u, \pi) + \ln(T))/\eta^\star\big)$. Using that $\lhat_t(i') \leq 2\ltil_t(i') = 2\ell_t(i)$ due to the restrictions on $\eta$, and that $\nabla^2 R_t(\p_t^{av})$ is a diagonal matrix with components equal to ${1}\big/{\eta p_t^{av}(i, \eta)}$, we can see that 
\begin{align*}
    & (\bLhat_t - \bLhat_t^{av})^\top \big(\nabla^{2} R_t(\p_t^{av})\big)^{-1} \blhat_t \\
    & \leq  \sum\nolimits_{i=1}^K \sum\nolimits_{\eta \in H_t} 4 \eta p_t^{av}(i, \eta) \ltil_t(i, \eta)\big(1 + \rho_t(i)\big) 
\end{align*}
Thus, by using the definition of $\u'$ and that $a_t(i, \eta) = 4 \eta \ltil_t(i, \eta) (1 + \rho_t(i))$, we derive a bound on $\sumT \inner{\q_t - \u}{\bell_t}$ of order
\[
    \min_{\eta\in H_T}\left(\frac{\KL(\u, \pi) + \ln(T)}{\eta} + \eta \sumT \ell_t(i)(1 + \rho_t(i))\right)
\]
To complete the proof one has to show that there is a suitable $\eta^\star \in H_T$ approximating the minimizer of the above expression (see Appendix \ref{app:Full information}).
%
%\begin{equation*}
%    \argmin_\eta \left\{\frac{\KL(\u, \pi) + \ln(T)}{\eta} + \eta \sumT \ell_t(i)(1 + \rho_t(i))\right\}
%\end{equation*}
\end{proof}
%
%As we mentioned in the introduction, there is an interesting trade-off between delays and losses. %If the arm with the lowest total loss has very large delay, the regret bound is also large. If in the same scenario there is an arm with slightly more total loss but less delay, it might be worth it to increase the loss of the comparator $\u$ at the gain of a smaller regret bound than an regret bound with respect to the best arm.
Interestingly, the result in Theorem \ref{th:informal fullinfobound} holds for all $\u$ simultaneously, which implies that the algorithm automatically performs a trade-off between losses and delays. This interplay between losses and delays also occurs when we run the algorithm: the loss $\lhat_t(i, \eta) = \ell_t(i)\big(1 + 4\eta(1 + \rho_t(i))\big)$ we feed to the algorithm penalizes arms for suffering large delay, as when the delay is large the amount of missing feedback is also large. 

\section{BANDIT SETTING}\label{sec:bandits}

In the bandit setting we face two additional challenges compared to the full information setting. The first challenge is to estimate $\bell_t$. For this we use the implicit exploration technique \citep{neu2015explore}: sample $i_t \sim \q_t$, observe $\ell_t(i_t)$, and estimate $\ell_t(i)$ for any $i$ by $\id[i_t = i](q_t(i) + \epsilon_t)^{-1}\ell_t(i)$, where $\epsilon_t \geq 0$ is a user-specified parameter which we will either set to $0$ or to $O(t^{-1/2})$.

The second challenge is due to arm-dependent delays. In Section~\ref{sec:Full information} we used $\rho_t(i)$ in the correction term. While ideally we would like to do the same in the bandit setting, we do not know how many observations are missing from each arm: if we do \emph{not} pull arm $i$ in round $t$ we will not know if $\ell_t(i)$ is missing in subsequent rounds. %will never know if any loss other than $\ell_t(i_t)$ is still missing in any of the subsequent rounds. 
% In the following sections, we provide instances of Algorithm~\ref{alg:framework} for the partially-concealed and concealed bandit settings.
Recall that in the partially concealed bandit setting we assume that the learner observes $\rho_t(i_t)$ whenever the loss $\ell_t(i_t)$ is observed. Note that in the arm-independent delay setting, $\rho_t(i)$ can always be computed by simply counting the number of missing losses.
%Note that in the setting where the delay is equal for each arm the delay of round $t$ is known for each arm as soon as the loss of round $t$ is observed by the learner. This implies that the learner can compute $\rho_t(i_t)$ by counting the number of missing losses. 
In the concealed bandit setting, we do not make this additional assumption.
The fact that we do not know $\rho_t(i_t)$ at prediction time complicates learning $\dTmax$. While in the full information setting we could adjust the grid of learning rates \emph{before} issuing a prediction, we can not use this trick in the bandit setting. Instead, we use a fixed grid of learning rates. %Note that in the special case where the delay is equal for each arm, it is in fact possible use a decreasing grid of learning rates to learn $\dTmax$, as $\rho_t(i)$ can be computed before issuing the prediction.
For this reason we assume that the learner has preliminary access to an upper bound on $\dTmax$, denoted by $\dstar$.

\subsection{Partially Concealed Bandit Setting}\label{sec:partcon}
As in the full information setting, we simultaneously learn the best learning rate and expert. % However, we only learn the best learning rate with respect to the overall cost of delay. 
For each expert $i$ we create $J$ pseudo-experts, each having its own personal learning rate $\gamma \in \Gamma$. Thus, as in the full information setting,
% we can identify a unique $i'$ for each pair $i$ and $\gamma$, and so
we write $p(i') = p(i, \gamma)$. Unlike before, we use Algorithm~\ref{alg:framework} with a double regularizer:
\begin{align}
    &\label{eq:banditreg1}
    \Psi_t(\p) = -\eta_t^{-1} \sum\nolimits_{i=1}^K \ln \left(\sum\nolimits_{\gamma \in \Gamma} p(i, \gamma)\right)\\
    &\label{eq:banditreg2}
    \Phi_t(\p) = \sum\nolimits_{i=1}^K \sum\nolimits_{\gamma \in \Gamma} \gamma^{-1}p(i, \gamma)  \ln\big(p(i, \gamma)\big)
\end{align}
The role of \eqref{eq:banditreg1} is to control the variance $\Eb{\blhat_t^\top \big(\nabla^{2} R_t(\p_t^{av})\big)^{-1} \blhat_t}$ and \eqref{eq:banditreg2} is used to control $\Eb{\big(\bLhat_t - \blhat_t - \bLhat_t^{av}\big)^\top \big(\nabla^{2} R_t(\p_t^{av})\big)^{-1} \blhat_t}$, a quantity associated with the additional regret due to delayed feedback.
%
% %
% \begin{equation}\label{eq:banditreg1}
%     \Psi_t(\p) = -\frac{1}{\eta_t} \sumK \ln \left(\sum_{\gamma \in \Gamma} p(i, \gamma)\right)
% \end{equation}
% %
% controls the variance $\Eb{\blhat_t^\top \big(\nabla^{2} R_t(\p_t^{av})\big)^{-1} \blhat_t}$ of the loss estimates $\blhat_t$, and 
% %
% \begin{equation}\label{eq:banditreg2}
%     \Phi_t(\p) = \sumK \sumgam \frac{p(i, \gamma)}{\gamma}  \ln\big(p(i, \gamma)\big)
% \end{equation}
% %
% controls $\Eb{\big(\bLhat_t - \blhat_t - \bLhat_t^{av}\big)^\top \big(\nabla^{2} R_t(\p_t^{av})\big)^{-1} \blhat_t}$, a quantity associated with the additional regret caused by delayed feedback.
% %
The predictions $i_t$ are sampled from $\q_t$ given by
\begin{equation}\label{eq:banditpredictions1}
    q_t(i) = \sum\nolimits_{\gamma \in \Gamma} p_t^{av}(i, \gamma)
\end{equation}
Our regularizer has a somewhat unusual structure. The log barrier regularizer~\eqref{eq:banditreg1} controls the variance of the loss estimates, and regularizes $q_t(i)$ rather than $p_t(i, \gamma)$. This implies that, as far as the log barrier is concerned, there are only $K$ arms rather than $K' = O\big(K\ln(T)\big)$ arms. To develop some intuition as to why this matters, let us consider the regret bound of FTRL run with the log barrier, a constant learning rate $\eta$, and $N$ arms. Without delays, this bound is $O\big({N\ln(T)}{\eta^{-1}} + \eta T\big)$ \citep{foster2016learning}. Therefore, using $K$ rather than $K'$ experts reduces the regret of FTRL with the log barrier by a factor of $\sqrt{\ln(T)}$.% if we use the optimal $\eta$. %The role of $\Phi_t$ is to regularize $p_t(i, \gamma)$, which we use to learn the optimal $\gamma$.

We continue by specifying %the importance weighted
loss estimates $\tilde{\ell_t}(i') = \tilde{\ell_t}(i, \gamma) = \id[i_t = i]q_t(i)^{-1}\ell_t(i)$ (i.e., we set $\epsilon_t = 0$) and the corrections
\begin{equation}\label{eq:correctionbandit}
    a_t(i, \gamma) = 4 \gamma \ltil_t(i, \gamma) \rho_t(i)
\end{equation}
Recall that in the partially concealed bandit setting the learner observes $\rho(i_t)$ whenever the loss $\ell_t(i_t)$ is observed, making \eqref{eq:correctionbandit} a valid choice for corrections. 
The grid $\Gamma$ of learning rates we use in this section contains all
\begin{equation}\label{eq:banditgamma}
    \min\bigg\{(4\dstar)^{-1}, \frac{\sqrt{\ln(K) + \ln(T) + 1}}{4\sqrt{\dstar}2^{j}}\bigg\}
\end{equation}
for $j = 1,\ldots,J = \lceil \log_2 (\sqrt{T}) \rceil$, where $\dstar \geq \dTmax$. The learning rate for the log barrier is
\begin{equation}\label{eq:banditeta}
    \eta_t = \sqrt{K\ln(T)\left(4(1 + \dstar) + 4\sum\nolimits_{s \in \Sset_t} \ell_s(i_s)\right)^{-1}}
\end{equation}
where $\Sset_t = \{s: s + d_s(i_s) < t\}$ is the set of available losses at the beginning of round $t$ in the bandit setting.
\begin{theorem}\label{th:informal banditbound}
Let $\p_1^{\Psi} \equiv \frac{1}{KJ}$ and let $\p_1^\Phi$ be given by~\eqref{eq:fullinfoprior}. Assume $\pi_1(i) \geq \frac{1}{T^2}$ for all $i\in [K]$.
If we run Algorithm \ref{alg:framework} using regularizers~\eqref{eq:banditreg1} and~\eqref{eq:banditreg2} with corresponding learning rates~\eqref{eq:banditgamma} and~\eqref{eq:banditeta}, then the predictions $i_t \sim \q_t$, with $\q_t$ as in \eqref{eq:banditpredictions1}, satisfy
%Let $\p_1^{\Psi} \equiv \frac{1}{KJ}$ and let $\p_1^\Phi$ be given by~\eqref{eq:fullinfoprior}. Assume $\pi_1(i) \geq \frac{1}{T^2}$ for all $i\in [K]$. With $\p_t^{av}$ from Algorithm \ref{alg:framework} with regularizers \eqref{eq:banditreg1} and \eqref{eq:banditreg2} and corresponding learning rates in equations \eqref{eq:banditgamma} and \eqref{eq:banditeta}, the predictions $i_t \sim \q_t$, with $\q_t$ as in \eqref{eq:banditpredictions1}, satisfy   
%
\begin{align*}
    & \E\Bigg[\sumT (\ell_t(i_t) - \inner{\u}{\bell_t})\bigg] 
    = O\Bigg(K\ln(T) + \dstar\ln(T) \\
    & \sqrt{K\ln(T)\inner{\u}{\L_T}} + \sqrt{\big(\KL(\u, \pi) + \ln(T)\big)\inner{\u}{\L_T^\rho}}\Bigg)
\end{align*}
\end{theorem}
The proof of Theorem~\ref{th:informal banditbound} follows from Theorem~\ref{th:banditbound} in Appendix \ref{app:partcon}. As in the full information setting, we use Lemma~\ref{lem:frameworkregret} to bound the regret. The main difference with the full information setting is how $\Eb{(\bLhat_t - \bLhat_t^{av})^\top \big(\nabla^{2} R_t(\p_t^{av})\big)^{-1} \blhat_t}$ is bounded. The unusual structure of the regularizer presents us with a challenge here, as $\nabla^2 R_t(\p_t^{av})$ is a complicated block-diagonal matrix. 
% A similar challenge was faced by \citet{jin2020simultaneously}, but for a different regularizer and purpose.   
We carefully analyse the inverse of $\nabla^2 R_t(\p_t^{av})$ in Lemmas \ref{lem:banditFbound} and \ref{lem:banditdriftbound}: 
\begin{restatable}{relemma}{banditFbound}
\label{lem:banditFbound}
For all $t$, the function $F_t = \Psi_t + \Phi_t$, where $\Psi_t$ is defined in~\eqref{eq:banditreg1} and $\Phi_t$ is defined in~\eqref{eq:banditreg2}, is Legendre %, strictly convex
and satisfies the conditions of Lemma~\ref{lem:drift}. Moreover,
\begin{align*}
     (\bLhat_t & - \bLhat_t^{av})^\top  \big(\nabla^{2} F_t(\p_t^{av})\big)^{-1} \blhat_t
\leq
    \eta_t 4\ell_t(i_t)^2
\\&
    + 4 \sum\nolimits_{\gamma \in \Gamma} \gamma \sum\nolimits_{s \in m_t(i_t)} p_t^{av}(i_t, \gamma) \ltil_t(i_t, \gamma) \ltil_s(i_t, \gamma)
\end{align*}
\end{restatable}
\begin{restatable}{relemma}{banditdriftbound}
\label{lem:banditdriftbound}
With regularizers \eqref{eq:banditreg1} and \eqref{eq:banditreg2} we have that 
\begin{equation*}
\begin{split}
    \E \bigg[(\bLhat_t - \bLhat_t^{av})^\top \big(\nabla^{2} R_t(\p_t^{av})\big)^{-1} \blhat_t\bigg] &  - \Eb{\inner{\p_t^{av}}{\a_t}} \\
    & \leq \Eb{\eta_t 4\ell_t(i_t)^2}
\end{split}
\end{equation*}
\end{restatable}
%\begin{proof}
\textit{Proof.} Since $\nabla^{2} R_t(\p_t^{av}) = \nabla^{2} F_t(\p_t^{av})$ and $\sum_{s \in m_t(i_t)} \ltil_t(i_t) f_s = \sumK \sum_{s=1}^T \id[s \in m_t(i)] \ltil_t(i)f_s$ for arbitrary $f_s$ we can use Lemma~\ref{lem:banditFbound} to show that
\begin{align}
\nonumber
    & \E\bigg[(\bLhat_t  - \bLhat_t^{av})^\top \big(\nabla^{2} R_t(\p_t^{av})\big)^{-1} \blhat_t\bigg]
\le
    \Eb{\eta_t 4\ell_t(i_t)^2} +
\\&
% \nonumber
%     + \Eb{4 \sum_{\gamma \in \Gamma} \sum_{s \in m_t(i_t)} p_t^{av}(i_t, \gamma) \gamma \ltil_t(i_t, \gamma) \ltil_s(i_t, \gamma)}
% \\ &\le
% \nonumber
%     \Eb{\eta_t 4\ell_t(i_t)^2}
%     + \E\Bigg[4 \sum_{\gamma \in \Gamma} \sumK \sum_{s=1}^T \id[s \in m_t(i)]
% \\&\qquad\qquad
% \label{eq:startbanditFbound}
%     \times p_t^{av}(i, \gamma) \gamma \ltil_t(i, \gamma) \ltil_s(i, \gamma) \Bigg]
\label{eq:startbanditFbound} \E\Bigg[4 \sum_{\gamma \in \Gamma} \sumK \sum_{s=1}^T \id[s \in m_t(i)] p_t^{av}(i, \gamma) \gamma \ltil_t(i, \gamma) \ltil_s(i, \gamma) \Bigg]
\end{align}
We proceed by studying the expectation of the sum in the above equation.
Denote by $\Bset_s \equiv \{i_r : r \in \Sset_s\}$ all $i_r$ with $r\in\Sset_s \equiv \{r: r + d_r(i_r) < s\}$. %, i.e. the set of available losses at the beginning of round $s$.
Recall that $m_t(i) \equiv \{s: s < t, s + d_s(i) \geq t\}$ identifies the set of losses that can not be used for prediction in round $t$ due to delay. %by the learner at the beginning of round $t$ for arm $i$ due to delay.
We may write
%
%\begin{align*}
%    & \E\Bigg[p_t^{av}(i, \gamma) \ltil_t(i, \gamma) \id[s \in m_t(i)]\ltil_s(i, \gamma) \Bigg| \Bset_s\Bigg] \\
%    & = \E\Bigg[p_t^{av}(i, \gamma) \ltil_t(i, \gamma) \ltil_s(i, \gamma) \Bigg| \Bset_s\Bigg] \id[s \in m_t(i)]
%\end{align*}
%
%Using that $\ell_t(i) \leq 1$ we can see that
% %
% \begin{align}\label{eq:lookingatproblem}
%     p_t^{av}(i, \gamma) \ltil_t(i, \gamma) \ltil_s(i, \gamma) \leq & p_t^{av}(i, \gamma)\ltil_t(i, \gamma) \frac{\id[i = i_s]}{q_s(i)}
% \end{align}
% %
%
%\begin{align}\label{eq:lookingatproblem}
%    p_t^{av}(i, \gamma) \ltil_s(i, \gamma) \leq p_t^{av}(i, \gamma) \frac{\id[i = i_s]}{q_s(i)}
%\end{align}
%
%where we used that $\ell_s(i) \leq 1$.
%
\begin{align}
\nonumber
    & \E\Bigg[p_t^{av}(i, \gamma) \ltil_t(i, \gamma) \id[s \in m_t(i)]\ltil_s(i, \gamma) \Bigg| \Bset_s\Bigg] \leq \\ %= \\
%\nonumber
%    & \E\Bigg[p_t^{av}(i, \gamma) \ltil_t(i, \gamma) \ltil_s(i, \gamma) \Bigg| \Bset_s\Bigg] \id[s \in m_t(i)] \le \\
% & \E\Bigg[p_t^{av}(i, \gamma) \ltil_t(i, \gamma) \id[s \in m_t(i)]\ltil_s(i, \gamma) \Bigg| \Bset_s\Bigg] \leq \\
\label{eq:tricky}
 & \E\Bigg[p_t^{av}(i, \gamma) \ltil_t(i, \gamma) \frac{\id[i = i_s]}{q_s(i)} \Bigg| \Bset_s\Bigg] \id[s \in m_t(i)] = \\
 \nonumber
&  \E\Bigg[p_t^{av}(i, \gamma) \ltil_t(i, \gamma)\Bigg| \Bset_s\Bigg] \E\Bigg[\frac{\id[i = i_s]}{q_s(i)} \Bigg| \Bset_s\Bigg] \id[s \in m_t(i)]
 % & = \E\Bigg[p_t^{av}(i, \gamma) \ltil_t(i, \gamma) \frac{\Pp\left(s \in m_t(i) \cap i = i_s\right)}{q_s(i)} \Bigg| \Bset_s \Bigg] \\
 %& = \E\Bigg[p_t^{av}(i, \gamma) \ltil_t(i, \gamma) \frac{\Pp\left(s \in m_t | i = i_s\right)\Pp\left(i = i_s\right)}{q_s(i_s)} \Bigg| \{i_r : r \in \Sset_s\}\Bigg] \\
 %& = \E\Bigg[p_t^{av}(i, \gamma) \ltil_t(i, \gamma) \Pp\left(s \in m_t | i = i_s\right) \Bigg| \{i_s : r \in \Sset_s\}\Bigg] 
\\ & =
    \Eb{p_t^{av}(i, \gamma)\ltil_t(i, \gamma)\Big|\Bset_s}\id[s \in m_t(i)]
\nonumber
\end{align}
The first step is true because $s \in m_t(i)$ is a deterministic event for all $s \le t$ and $i$ %. The second step 
and simply uses $\ell_s(i) \leq 1$. Equality~\eqref{eq:tricky} vacuously holds when $s \not\in m_t(i)$. If $s \in m_t(i)$, then $\ltil_s(i, \gamma)$ is not used for computing the prediction in round $t$ because it is not available at the beginning of round $t$.
% (or any rounds between $s$ and $t$).
Hence, when $i = i_s$, $p_t^{av}(i, \gamma)\ltil_t(i, \gamma)$ does not depend on $i_s$ % in $\Bset_s$
and~\eqref{eq:tricky} holds; if $i \neq i_s$ then both sides of the equality are again zero and~\eqref{eq:tricky} vacuously holds. 
%Here we may observe something critical. Recall that
% $s$ is not in $\Sset_t$, i.e.
%On the other hand, whenever $\id[i = i_s] = 0$, the right-hand side of~\eqref{eq:lookingatproblem}
% is equal to 0. %In other words, conditioned on the losses received before round $s$, when computing the expectation of $p_t^{av}(i, \gamma)\ltil_t(i, \gamma)\id[s \in m_t(i)]\ltil_s(i, \gamma)$ we may treat $p_t^{av}(i, \gamma)\ltil_t(i, \gamma)$ and $\id[s \in m_t(i)]\ltil_s(i, \gamma)$ as independent random variables.
% conditioned on all the feedback observed before round $t$ and $i_t$,
%
The above implies that we can further bound equation \eqref{eq:startbanditFbound} as 
\begin{equation*}
\begin{split}
    \E \bigg[(\bLhat_t & - \bLhat_t^{av})^\top \big(\nabla^{2} R_t(\p_t^{av})\big)^{-1} \blhat_t\bigg] \leq \Eb{\eta_t 4\ell_t(i_t)^2} \\
    & + \Eb{4 \sum\nolimits_{\gamma \in \Gamma} p_t^{av}(i_t, \gamma) \gamma \ltil_t(i_t, \gamma)\rho_t(i_t)}
\end{split}
\end{equation*}
which completes the proof after we subtract
\[
\Eb{\inner{\p_t^{av}}{\a_t}} = \Eb{4 \sum\nolimits_{\gamma \in \Gamma} p_t^{av}(i_t, \gamma) \gamma \ltil_t(i_t, \gamma)\rho_t(i_t)} ~\square
\]
% To complete the proof we subtract $\Eb{\inner{\p_t^{av}}{\a_t}} = \Eb{4 \sumK \sum_{\gamma \in \Gamma} p_t^{av}(i_t, \gamma) \gamma \ltil_t(i_t, \gamma)\rho_t(i_t)} $ from the above equation:
% %
% \begin{equation*}
% \begin{split}
%     \E \bigg[(\bLhat_t - \bLhat_t^{av})^\top (\nabla^{2} R_t(\p_t^{av}))^{-1} \blhat_t\bigg] &  - \Eb{\inner{\p_t^{av}}{\a_t}} \\
%     & \leq \Eb{\eta_t 4\ell_t(i_t)^2}.
% \end{split}
% \end{equation*}
%
%\end{proof}
%
Lemma~\ref{lem:banditdriftbound} combined with the bound on $R_T(\u')$ in Lemma~\ref{lem:banditRbound} and showing that there is a $\gamma \in \Gamma$ that is close to optimal are the essential parts of the proof of Theorem~\ref{th:banditbound}, which can be found in Appendix~\ref{app:partcon}.

\subsection{Concealed Bandit Setting} \label{sec:concealed}
In this section we design an algorithm that does not require knowledge of $\rho_t(i)$ for any arm. To see why not knowing $\rho_t(i)$ poses a problem for standard algorithms, %in the delayed multi-armed bandit setting, note that the optimal algorithm of
we analyse the optimal algorithm of \citet{zimmert2020optimal}. This
%. The algorithm of \citet{zimmert2020optimal}
corresponds to Algorithm \ref{alg:framework} with the following setup:
The algorithm uses $i'= i$, which is to say that it samples its actions from %
$
    q_t(i) = p_t^{av}(i)
$
where $p_t^{av}(i)$ are computed by Algorithm \ref{alg:framework}. Furthermore, it uses the following regularizers: 
\begin{align}
    & \label{eq:regunknown1} \Psi_t(\p) = \sum\nolimits_{i =1}^K-\eta_t^{-1}\sqrt{p(i)}\\
    & \label{eq:regunknown2} \Phi_t(\p) = \sum\nolimits_{i =1}^K \gamma_t^{-1}p(i)\ln\big(p(i)\big)
\end{align}
% %
% \begin{equation}\label{eq:regunknown1}
%     \Psi_t(\p) = \sumK-\frac{1}{\eta_t}\sqrt{p(i)}
% \end{equation}
% %
% to control the variance of the loss estimates and 
% %
% \begin{equation}\label{eq:regunknown2}
%     \Phi_t(\p) = \sumK \frac{1}{\gamma_t}p(i)\ln\big(p(i)\big)
% \end{equation}
% %
% to control the cost of delay. 
\citet{zimmert2020optimal} use $a_t(i) = 0$ and we postpone specifying the loss estimates until later. 
Recall that $F_t = \Psi_t + \Phi_t$ and observe that $\big(\nabla F_t(\x)\big)_i = -\frac{1}{2\eta_t\sqrt{x(i)}} + \frac{\ln(x(i)) + 1}{\gamma_t}$ and that $\nabla^2 F_t(\x)$ is a diagonal matrix with $\big(\nabla^2 F_t(\x)\big)_{i, i} = -\frac{1}{4\eta_t x(i)^{3/2}} + \frac{1}{\gamma_t x(i)}$. 
As the conditions of Lemma~\ref{lem:frameworkregret} are satisfied, we may write:
\begin{align}
\nonumber
&
     \Eb{\sumT \inner{\q_t - \u}{\bltil_t}} \leq  \E\Bigg[\sumT \sumK  4 \eta_t q_t(i)^{3/2} \ltil_t(i)^2\Bigg]
\\&
\label{eq:mintrick}
    + \Eb{R_T(\u)} + \E\Bigg[\sumT \sumK  \gamma_t q_t(i)\ltil_t(i)\sum_{s \in m_t(i)}\ltil_s(i)\Bigg]
\end{align}
%
% \begin{align}
% \nonumber
% &
%      \Eb{\sumT \inner{\q_t - \u}{\bltil_t}} \leq  \Eb{R_T(\u)}
% \\&
% \nonumber
%     + \Eb{\sumT \sumK  \frac{\ltil_t(i)\left(\ltil_t(i) + \sum_{s \in m_t(i)} \ltil_s(i)\right)}{\frac{1}{4 \eta_t q_t(i)^{3/2}} + \frac{1}{\gamma_t q_t(i)}}}
% \\& \le
% \nonumber
%     \Eb{R_T(\u)} + \E\Bigg[\sumT \sumK  4 \eta_t q_t(i)^{3/2} \ltil_t(i)^2\Bigg]
% \\&
% \label{eq:mintrick}
%     + \E\Bigg[\sumT \sumK  \gamma_t q_t(i)\ltil_t(i)\sum_{s \in m_t(i)}\ltil_s(i)\Bigg]
% \end{align}
%
where we used that
$\left(\frac{1}{4 \eta_t q_t(i)^{3/2}} + \frac{1}{\gamma_t q_t(i)}\right)^{-1} \leq \min\{4 \eta_t q_t(i)^{3/2}, \gamma_t q_t(i)\}$.
%$\big(\frac{1}{a} + \frac{1}{b}\big)^{-1} \le \min\{a,b\}$ for $a,b > 0$.
Following the proof of Lemma~\ref{lem:banditdriftbound} to bound the expectation, we see that
\begin{align*}
    \Eb{\sumK q_t(i) \ltil_t(i)\sum_{s \in m_t(i)} \ltil_s(i)} \leq \Eb{\sumK q_t(i) \rho_t(i)}
\end{align*}
Since \citet{zimmert2020optimal} assume that the delays are equal for each arm, implying that $\rho_t(i) = \rho_t$ for all arms, they can use an unbiased loss estimator combined with $\gamma_t = \sqrt{\frac{\ln(K)}{\sum_{s=1}^{t} \rho_t}}$ to achieve the optimal regret bound. In an ideal scenario, in our setting we would set $\gamma_t = \sqrt{\frac{\ln(K)}{\sum_{s=1}^{t} \sumK \gamma_t q_t(i) \rho_t(i)}}$. Unfortunately, we do not know $\rho_t(i)$, meaning we have to resort to tuning $\gamma_t$ in terms of (a suitable upper bound on) $\sumK  \gamma_t q_t(i)\ltil_t(i)\sum_{s \in m_t(i)}\ltil_s(i)$. This brings us an additional challenge in the form of having to control the loss estimates, which we do by using the implicit exploration of \citet{neu2015explore}, as is also done by \citet{gyorgy2020adapting}, i.e., we set $\ltil_t(i) = \id[i_t = i]\ell_t(i)(q_t(i) + \epsilon_t)^{-1}$. %$\ltil_t(i) = \frac{\id[i_t = i]\ell_t(i)}{q_t(i) + \epsilon_t}$.
The role of $\epsilon_t$ is to control the range of $\ltil(i)$, which is $[0, \epsilon_t^{-1}]$.
The learning rate for $\Phi_t$ will be set as
\begin{equation}\label{eq:censoredgamma}
    \gamma_t = \sqrt{\frac{\ln(K)}{\frac{\dstar}{\epsilon_t} + \sum_{s=1}^{t-1} \sumK \id[i_s = i]\sum_{s' \in m_s(i)}\frac{\id[i_{s'} = i]}{q_{s'}(i)}}}   
\end{equation}
Note that while we do not know $m_s(i) \equiv \{s: s < t, s + d_s(i) > t\}$, the set of missing losses for arm $i$ at time $t$, we can indeed compute $\sum_{s \in m_s(i)}\frac{\id[i_{s} = i]}{q_{s}(i)}$, as we only add $\frac{1}{q_{s}(i)}$ to the sum whenever we pulled arm $i$ in round $s$ and have not yet observed $\ell_s(i_s)$ by round $t$, which we do know. The final result can be found in Theorem \ref{th:concealed}, of which the proof can be found in Appendix~\ref{app:concealed}.
\begin{restatable}{retheorem}{thconcealed}\label{th:concealed}
Let $\q_t$ be computed by Algorithm \ref{alg:framework} with $a_t(i) = 0$, $\eta_t = \frac{1}{\sqrt{4 t}}$, $\epsilon_t = \frac{1}{\sqrt{t}}$, $\gamma_t$ as in \eqref{eq:censoredgamma}, $p_1^\Psi(i) = p_1^\Phi(i) = \frac{1}{K}$, $\ltil_t(i) = \frac{\id[i_t = i]\ell_t(i)}{q_t(i) + \epsilon_t}$, and regularizers specified in %equations
\eqref{eq:regunknown1} and \eqref{eq:regunknown2}. Then actions $i_t \sim \q_t$ guarantee
\begin{align*}
    & \Eb{\sum\nolimits_{t=1}^T  (\ell_t(i_t) - \inner{\u}{\bell_t}}) 
    \leq  9 \sqrt{KT} + \half \dstar \\
    & \quad + 3\sqrt{\ln(K)\sum\nolimits_{t=1}^T\Eb{\sum\nolimits_{i = 1}^K q_t(i) \rho_t(i)}} 
\end{align*}
\end{restatable}

\section{FUTURE WORK}\label{sec:conclusion}

Our analysis of Lemma~\ref{lem:frameworkregret} (and of Lemma~\ref{lem:drift} in particular) hinges on the non-negativity of the losses. While for first-order bounds non-negativity is a standard assumption, for other types of bounds a typical assumption is $\ell_t(i) \in [-1, 1]$. Since most works in the delayed feedback setting also assume and use non-negativity of the losses, we think that deriving an algorithm and a corresponding analysis without relying on the non-negativity of the losses would be an interesting contribution. Another interesting direction to pursue is developing a skipping procedure for algorithms that are adaptive to both data and delay. While several authors have proposed different procedures for skipping rounds with a large delay \citep{thune2019nonstochastic, zimmert2020optimal, gyorgy2020adapting}, it is not clear how to apply these techniques to the arm-dependent delay setting. With regard to a skipping technique for data and delay adaptive algorithms, as \citet{gyorgy2020adapting} note, applying standard skipping techniques induces a complicated dependence on past actions, which significantly complicates the analysis.

\paragraph{Acknowledgements}
Nicol\`o Cesa-Bianchi and Dirk van der Hoeven gratefully acknowledge support by the MIUR PRIN grant Algorithms, Games, and Digital Markets (ALGADIMAR). Nicol\`o Cesa-Bianchi was also supported by the EU Horizon 2020 ICT-48 research and innovation action under grant agreement 951847, project ELISE.

% \section{EXPERIMENTS} 

% If you want to implement just use cvxopt. 

% \subsection{Fast Implementation}

% \section{FUTURE WORK}\label{sec:conclusion}

% Our analysis of Lemma~\ref{lem:frameworkregret} (and of Lemma~\ref{lem:drift} in particular) hinges on the non-negativity of the losses. While for first-order bounds non-negativity is more common and desirable assumption, for other types of bounds a more common assumption is that $\ell_t(i) \in [-1, 1]$. Since most works in the delayed feedback setting also assume and use non-negativity, we think that deriving an algorithm and corresponding analysis which does not use the non-negativity of the losses would be an interesting contribution. Another interesting direction to pursue is deriving a skipping procedure for the arm-dependent delay setting and for algorithms that are adaptive to both data and delay. While several authors have proposed different procedures for skipping rounds which have a large delay \citep{thune2019nonstochastic, zimmert2020optimal, gyorgy2020adapting}, it is not clear how to apply these techniques to the arm-dependent delay setting. With regard to a skipping technique for data and delay adaptive algorithms, as \citet{gyorgy2020adapting} note, applying standard skipping techniques induces a complicated dependence on past actions, which significantly complicates the analysis. 

\subsubsection*{References}
\begingroup
\renewcommand{\section}[2]{}
\bibliography{abib.bib}
\endgroup

\newpage

\onecolumn
\appendix

\section{DETAILS OF SECTION \ref{sec:algorithm}}\label{app:algorithm}

We start with some definitions and a technical Lemma. Let $\L_0^t = \nabla \Psi_t(\p^\Psi_1) + \nabla \Phi_t(\p^\Phi_1)$ and
\[
    \Phisttr = \min_{\p \in \triangle'} \inner{\p}{\cdot} + F_t(\p)
\]

\begin{lemma}\label{lem:conjugateproperties3}
% Suppose that $F_t = \Psi_t + \Phi_t$ is a Legendre function and that for any $\z$ with non-negative elements, $H_t(\x, \z) = \inner{\nabla F_t(\x)}{\z}$ is concave in $\x$ and $H_t(\x_1, \z) \leq H_t(\x_2, \z)$ iff $\x_1 \leq \x_2$.
% Let $R_t(\p) = B_{\Psi_t}(\p, \p^\Psi) + B_{\Phi_t}(\p, \p^\Phi)$ and let $\L_0 = \nabla \Psi_t(\p^\Psi_1) + \nabla \Phi_t(\p^\Phi_1)$.
Under the assumptions of Lemma~\ref{lem:drift}, for any $\p^{\L} = \argmin_{\p \in \triangle'} \inner{\p}{\L} + R_t(\p)$ there exists a $c \in \reals$ such that
\begin{align*}
    \p^{\L} = \nabla \Phisttr_t(-\L - \L_0^t) = \nabla \Phist_t(-\L- \L_0^t + c\1)
\end{align*}
Furthermore, for any vector $\z > \0$, $G_t(\cdot,\z)$
% = \inner{\nabla \Phi_t^\star(\cdot)}{\z}$
is a convex function, $G_t(\y_1, \z) \leq G_t(\y_2, \z)$ if $\y_1 \leq \y_2$, and $G_t(\y_1, \z) < G_t(\y_2, \z)$ if $\y_1 < \y_2$. 
\end{lemma}
\begin{proof}
Recall that $F_t = \Psi_t + \Phi_t$. To prove that $\p^{\L} = \nabla \Phisttr_t(-\L - \L_0) = \nabla \Phist_t(-\L- \L_0^t + c\1)$, observe that we may rewrite 
\begin{align*}
     \p^{\L} & = \argmin_{\p \in \triangle'} \inner{\p}{\L} + R_t(\p) \\
    & =  \argmin_{\p \in \triangle'} \inner{\p}{\L} + F_t(\p) - \Psi_t(\p_1^\Psi) -  \Phi_t(\p_1^\Phi) \\
    & - \inner{\nabla \Psi_t(\p_1^\Psi)}{\p - \p_1^\Psi} - \inner{\nabla \Phi_t(\p_1^\Phi)}{\p - \p_1^\Phi} \\
    & =  \argmin_{\p \in \triangle'} \inner{\p}{\L + \L_0^t} + F_t(\p)
\end{align*}
Following \citet[Fact 3]{zimmert2020optimal}, by the KKT conditions there exists a $c \in \reals$ such that $\p^\L$ satisfies $\nabla F_t(\p^\L) = -\L - \L_0^t + c\1$. Using that $\nabla F_t = (\nabla \Phist_t)^{-1}$, due to $F_t$ being Legendre, completes the proof of the first property.

In order to prove the second part, suppose now $\y_1 < \y_2$. Since $F_t$ is Legendre, $\y_1 = \nabla F_t\big(\nabla F_t^\star(\y_1)\big)$ and we have that 
\begin{align*}
    \inner{\nabla F_t(\nabla F_t^\star(\y_1))}{\1} < \inner{\nabla F_t(\nabla F_t^\star(\y_2))}{\1}
\end{align*}
Since $\inner{\nabla F_t(\x_1)}{\1} < \inner{\nabla F_t(\x_2)}{\1}$ iff $\x_1 < \x_2$ holds by hypothesis, the above implies $\nabla F_t^\star(\y_1) < \nabla F_t^\star(\y_2)$, which is equivalent to $G_t(\y_1, \1) < G_t(\y_2, \1)$ if $\y_1 < \y_2$. Using the same proof technique, replacing the strict inequality with an inequality, one can prove that $G_t(\y_1, \z) \leq G_t(\y_2, \z)$ if $\y_1 \leq \y_2$

As for the convexity of $G_t(\cdot,\z) = \inner{\nabla F_t^\star(\cdot)}{\z}$, let $\a = \nabla F_t^\star(p\,\y_1 + (1 - p) \y_2)$ and $\b = p \nabla F_t^\star(\y_1) + (1 - p) \nabla F_t^\star(\y_2)$ for $0 \le p \le 1$. By the concavity of $W_t$, we have that 
% https://math.stackexchange.com/questions/691796/convexity-of-inverse-function
\begin{align*}
    & W_t(\a, \z)
    = \inner{\nabla F_t(\a)}{\z} \\
    & = \inner{p\,\y_1 + (1 - p) \y_2}{\z} \\
    & = \inner{p \nabla F_t(\nabla F_t^\star(\y_1)) + (1 - p) \nabla F_t(\nabla F_t^\star(\y_2))}{\z} \\
    & = p W_t(\nabla F_t^\star(\y_1)), \z) + (1 - p) W_t(\nabla F_t^\star(\y_2)), \z) \\
    & \leq W_t(p \nabla F_t^\star(\y_1)) + (1 - p) \nabla F_t^\star(\y_2)), \z) \\
    & = W_t(\b, \z)
\end{align*}
Since $W_t(\x_1, \z) \leq W_t(\x_2, \z)$ iff $\x_1 \leq \x_2$, we must have that $\a \leq \b$, and thus
\begin{align*}
    p\,G_t(\y_1, \z) + (1 - p)G_t(\y_2, \z)
    \geq G_t(p\,\y_1 + (1 - p)\y_2, \z)
\end{align*}
which completes the proof. 
\end{proof}

\lemdrift*
\begin{proof}
We start by using Lemma \ref{lem:conjugateproperties3}:
\begin{align*}
    \inner{\p_t^{av} -\p_{t+1} }{\blhat_t}
%    &  \inner{\nabla \Phisttr_t(-\bLhat_t^{av} - \L_0^t) -\nabla \Phisttr_t(-\bLhat_t - \L_0^t)}{\blhat_t} = \\
& =
    \inner{\nabla F_t^\star(-\bLhat_t^{av}-\L_0^t + c_t^{av}\1) 
    - \nabla F^\star_t(-\bLhat_t - \L_0^t + c_t\1)}{\blhat} 
\\&
    = G_t(-\bLhat_t^{av} - \L_0^t + c_t^{av}\1, \blhat_t) - G_t(-\bLhat_t - \L_0^t + c_t\1, \blhat_t)
\end{align*}
Now, we claim that $c^{av}_t \leq c_t$, which we prove by contradiction. First, note that in the case where $\bLhat^{av}_t = \bLhat_t$ we have that $c^{av}_t = c_t$, so we only need to prove that $c^{av}_t \leq c_t$ when $\bLhat^{av}_t \not = \bLhat_t$, in which case $\bLhat^{av}_t < \bLhat_t$ must hold since $\lhat_s(i) \geq 0$. Suppose that $c^{av}_t > c_t$. We have that
\begin{align*}
    1
&=
    \inner{\p_t^{av}}{\1}
% & =
%    \inner{\nabla \Phisttr_t(-\bLhat_t^{av}- \L_0^t)}{\1}
\\& =
    G_t(-\bLhat_t^{av}- \L_0^t + c_t^{av}\1, \1)
% \\& >
%    G_t(-\bLhat_t^{av} - \L_0^t + c_t \1, \1)
% && \text{Lemma \ref{lem:conjugateproperties3} and $c^{av}_t > c_t$} \\
\\& >
    G_t(-\bLhat_t- \L_0^t  + c_t \1, \1)
% && \text{Lemma \ref{lem:conjugateproperties3} and $\Lhat_t(i) \geq \Lhat_t^{av}(i)$} \\
\\& =
    \inner{\nabla \Phist_t(-\bLhat_t - \L_0^t + c_t \1)}{\1}
\\& =
    \inner{\p_{t+1}}{\1}
=
    1
\end{align*}
where in the first inequality we used Lemma \ref{lem:conjugateproperties3}, the assumption that $c^{av}_t > c_t$,
% for the second inequality we have used Lemma \ref{lem:conjugateproperties3} again
and the fact that $\bLhat_t > \bLhat_t^{av}$.
For the penultimate equality we used that $\nabla \Phist_t(-\bLhat_t - \L_0^t + c_t \1) = \p_{t+1}$ by Lemma \ref{lem:conjugateproperties3}.
The last equation provides a contradiction, allowing us to conclude that $c^{av}_t \leq c_t$. Using Lemma~\ref{lem:conjugateproperties3} we can then write
\begin{align*}
\inner{\p_t^{av} -\p_{t+1} }{\blhat}
%\\& =
%    \inner{\nabla \Phist_t(-\bLhat_t^{av} - \L_0^t + c_t^{av} \1) - \nabla \Phist_t(-\bLhat_t - \L_0^t + c_t \1)}{\blhat}
& =
    G_t(-\bLhat_t^{av}  - \L_0^t + c_t^{av}\1, \blhat_t) - G_t(-\bLhat_t - \L_0^t + c_t\1, \blhat_t)
\\& \leq
    G_t(-\bLhat_t^{av} - \L_0^t + c_t^{av}\1, \blhat_t) - G_t(-\bLhat_t - \L_0^t + c_t^{av}\1, \blhat_t) %&& \text{Lemma \ref{lem:conjugateproperties3} and $c^{av} \leq c$} \\
\\& \leq
    \big(\bLhat_t - \bLhat_t^{av}\big)^\top \nabla G_t(-\bLhat_t^{av} - \L_0^t + c_t^{av}\1, \blhat_t)
%\\& =
%    (\bLhat_t - \bLhat_t^{av})^\top  \nabla^2 F_t^{\star}(-\bLhat_t^{av} - \L_0^t + c_t^{av}\1) \blhat_t
\end{align*}
where in the last step we used the convexity of $G_t$ in its first argument (Lemma \ref{lem:conjugateproperties3}). Recalling the definition of $G_t$, we have $\nabla G_t(\y,\z) = \nabla^2 F_t^\star(\y)\z$, implying
\[
    \inner{\p_t^{av} -\p_{t+1} }{\blhat}
\le
    \big(\bLhat_t - \bLhat_t^{av}\big)^\top  \nabla^2 F_t^{\star}(-\bLhat_t^{av} - \L_0^t + c_t^{av}\1) \blhat_t
\]
Since $F$ is Legendre, $\nabla F^{\star}$ is the inverse function of $\nabla F$ and we can use the inverse function theorem stating that $\nabla^2 F_t^{\star}(\x) = \big(\nabla^2 F_t(\nabla F_t^{\star}(\x))\big)^{-1}$. Hence
\begin{align*}
    \inner{\p_t^{av} -\p_{t+1} }{\blhat}
& \le
    \big(\bLhat_t - \bLhat_t^{av}\big)^\top  \nabla^2 F_t^{\star}(-\bLhat_t^{av} - \L_0^t + c_t^{av}\1) \blhat_t
\\& =
    \big(\bLhat_t - \bLhat_t^{av}\big)^\top \Big(\nabla^{2} F_t\big(\Phist_t(-\bLhat_t^{av} - \L_0^t + c_t^{av} \1)\big)\Big)^{-1} \blhat_t
\\& =
    \big(\bLhat_t - \bLhat_t^{av}\big)^\top \big(\nabla^{2} F_t(\p_t^{av})\big)^{-1} \blhat_t
\end{align*}
where in the last step we used $\p_t^{av} = \nabla \Phist_t(-\bLhat_t^{av} - \L_0^t + c_t^{av} \1)$. 
The proof is concluded by observing that $\nabla^2 F_t = \nabla^2 R_t$. 
\end{proof}
%
%
%\todo{Find a good location to define $\Phisttr$}
%

%
\lemcheating*
\begin{proof}
We start by proving by induction that $\sum_{t=1}^{T} \inner{\p_{t+1}}{\blhat_t} + R_{1}(\p_0) \leq \sum_{t=1}^{T} \inner{\p}{\blhat_t} + R_{T}(\p)$ holds for any $\p \in \triangle$ and $T$. In the base case $T = 0$, for which the inequality holds by definition of $\p_0$, as for $T = 0$ the sum is empty and $R_1 = R_0$. To prove the induction step, we fix $T > 0$ and assume that for any $\p\in \triangle'$
\begin{align*}
    & \sum_{t=1}^{T-1} \inner{\p_{t+1}}{\blhat_t} + R_{1}(\p_0) \leq \sum_{t=1}^{T-1} \inner{\p}{\blhat_t} + R_{T-1}(\p)
\end{align*}
Adding $\inner{\p_{T+1}}{\blhat_{T}}$ to both sides of the equation above we find
\begin{align*}
    \inner{\p_{T+1}}{\blhat_{T}} + \sum_{t=1}^{T-1} \inner{\p_{t+1}}{\blhat_t}+ R_{1}(\p_0)  
    & \leq \inner{\p_{T+1}}{\blhat_{T}} + \sum_{t=1}^{T-1} \inner{\p}{\blhat_t} + R_{T-1}(\p).
\end{align*}
Choosing $\p = \p_{T+1}$, using that $R_{T-1}(\p) \leq R_T(\p)$, and using the definition of $\p_{T+1}$ we see that 
\begin{align*}
    \sum_{t=1}^{T} \inner{\p_{t+1}}{\blhat_t}  + R_{1}(\p_0) 
    & \leq \sum_{t=1}^{T} \inner{\p_{T+1}}{\blhat_t} + R_{T-1}(\p_{T+1})\\
    & \leq \sum_{t=1}^{T} \inner{\p_{T+1}}{\blhat_t} + R_{T}(\p_{T+1})\\
    & = \min_{\p \in \triangle} \left\{\sum_{t=1}^{T} \inner{\p}{\blhat_t} + R_{T}(\p)\right\}
\end{align*}
which proves the inductive step. 

The above implies that 
\begin{align*}
    & \sum_{t=1}^{T} \inner{\p_{t+1}}{\blhat_t}  + R_{1}(\p_0)  \leq \sum_{t=1}^{T} \inner{\u}{\blhat_t} + R_{T}(\u)
\end{align*}
Finally, by reordering the equation above we find
\begin{align*}
    \sum_{t=1}^{T} \inner{\p_{t+1} - \u}{\blhat_t} \leq R_{T}(\u) - R_1(\p_0)
\end{align*}
concluding the proof after observing that $R_1(\p_0)$ since the Bregman divergence is non-negative.
\end{proof}

\section{DETAILS OF SECTION \ref{sec:Full information}}\label{app:Full information}

\begin{theorem}\label{th:fullinfobound}
%For any $\u \in \triangle$, $\q_t$ specified in \eqref{eq:fullinfopred} with $\p_t^{av}$ from Algorithm \ref{alg:framework} with the regularizer in \eqref{eq:fullinforeg}, corrections in \eqref{eq:fullinfocorr}, and prior in \eqref{eq:fullinfoprior} satisfies:
If we run Algorithm~\ref{alg:framework} with regularizer $\Phi_t \equiv 0$, $\Psi_t$ defined by~\eqref{eq:fullinforeg}, corrections $a_t$ defined by~\eqref{eq:fullinfocorr}, and prior defined by~\eqref{eq:fullinfoprior}, then the predictions $\q_t$ defined by~\eqref{eq:fullinfopred} using $\p_t^{av}$ returned by the algorithm satisfy
\begin{align*}
     \sumT   \inner{\q_t - \u}{\bell_t} & \leq 4(1 + \dTmax) + 12\sqrt{1+\dTmax}   + 8 \sqrt{\big(\KL(\u, \pi) + 2(\ln(T) + 1)\big)\inner{\u}{\L_T + \L_T^\rho}} \\
    & + 8\big(\KL(\u, \pi) + 2(\ln(T) + 1)\big)(1 + \dTmax)  + 16\sqrt{(1 + \dTmax)\big(\ln(K) + 2(1 + \ln(T))\big)}
\end{align*}
for any $\u \in \triangle$.
\end{theorem}
\begin{proof}[Proof of Theorem \ref{th:fullinfobound}]
We consider the enlarged simplex $\Delta'$ over $K' = K \times J$ coordinates. Fix any $\u'\in\Delta'$.
First, observe that by definition of $\q_t$, $\u'$, and $\bltil$ we have that 
\begin{equation*}
\begin{split}
    & \sumT \inner{\q_t - \u}{\bell_t} = \sumT \inner{\p_t^{av} - \u'}{\bltil_t} 
\end{split}
\end{equation*}
where the inner product on the left-hand side is on $\reals^{K}$ and the inner products on the right-hand side are on $\reals^{K \times J}$.

It is relatively easy to verify that, for any $t$,
\begin{align*}
    R_T(\u')
&=
    \sumK \sum_{\eta \in H_T} \Bigg(\frac{1}{\eta}u'(i, \eta) \ln\left(\frac{u'(i, \eta)} {p_1^{\Psi}(i, \eta)}\right)
    + \frac{1}{\eta}\big(p_1^\Psi(i, \eta) - u'(i, \eta)\big)\Bigg)
\end{align*}
verifies the conditions of Lemmas~\ref{lem:drift} and~\ref{lem:cheating regret}.
Indeed,
\begin{align*}
    W_t(\u', \z)
&=
    \inner{\nabla F_t(\u')}{\z}
    = \sumK \sum_{\eta \in H_T} \left( \frac{z(i,\eta)}{\eta}\ln\big(u'(i,\eta)\big) + 1\right)
\end{align*}
is concave and strictly monotone in the coefficients $u'(i,\eta)$ as required by Lemma~\ref{lem:drift}. To see that $F_t$ is Legendre, note that $\nabla^2 F(\x)$ is positive definite and thus $F_t$ is strictly convex. Now pick any sequence $\x_1, \x_2, \ldots$ in the interior of the domain of $F_t$ converging to a boundary point. $\|\nabla F_t(\x_n)\| \to \infty$ because $(\nabla F_t(\x_n))_{i'}$ is increasing in $x_n(i')$.
Moreover, the learning rates defined in~\eqref{eq:fullinfoeta} are decreasing in $t$ and the Bregman divergence is non-negative, implying $R_t(\u') \geq R_{t-1}(\u')$ as required by Lemma~\ref{lem:cheating regret}.
Therefore we can apply Lemma~\ref{lem:frameworkregret} to $R_t$ and obtain:
\begin{equation}\label{eq:thfullinfo1}
\begin{split}
    & \sumT \inner{\q_t - \u}{\bell_t} \leq R_T(\u')  + \sumT \inner{\u'}{\a_t}  + \sumT \Big((\bLhat_t - \bLhat_t^{av})^\top \big(\nabla^{2} R_t(\p_t^{av}))\big)^{-1} \blhat_t - \inner{\p_t^{av}}{\a_t}\Big)
\end{split}
\end{equation}
We continue by bounding $R_T$. We have that 
\begin{equation}\label{eq:expanded R_T fullinfo}
\begin{split}
    R_T(\u')
&=
    \sumK \sum_{\eta \in H_T} \Bigg(\frac{1}{\eta}u'(i, \eta) \ln\left(\frac{u'(i, \eta)} {p_1^{\Psi}(i, \eta)}\right)
    + \frac{1}{\eta}\big(p_1^\Psi(i, \eta) - u'(i, \eta)\big)\Bigg)
\\& 
    \leq \sumK \sum_{\eta \in H_T} \frac{1}{\eta}u'(i, \eta) \ln\left(\frac{u'(i, \eta)} {p_1^{\Psi}(i, \eta)}\right)
    + \sumK \sum_{\eta \in H_T}  \frac{1}{\eta} p_1^\Psi(i, \eta)
% \\&=
%     \frac{1}{\eta^\star} \KL(\u, \pi)  +  \frac{1}{\eta^\star} \ln\left(\frac{\sum_{j' = 1}^J 2^{-2j'}}{2^{-2j^\star}}\right)
% \\&
%     + \sumK \sum_{\eta \in H_T} \frac{1}{\eta}\big(p_1^\Psi(i, \eta) - u'(i, \eta)\big)
% \\&\leq
%     \frac{\KL(\u, \pi) + 2j^\star}{\eta^\star}
%     + \sumK \sum_{\eta \in H_T} \frac{1}{\eta} p_1^\Psi(i, \eta)
% \\&\leq
%     \frac{\KL(\u, \pi) + 2J}{\eta^\star}
%     + \sumK \sum_{\eta \in H_T} \frac{1}{\eta}p_1^\Psi(i, \eta)
\end{split}
\end{equation}
We proceed by bounding the first sum on the right-hand side of equation \eqref{eq:expanded R_T fullinfo}.
Denote by $\eta^\star \in H_T$ a target learning rate and let $j^\star$ be the corresponding index of $\eta^{\star}$. Let $\u'\in\Delta'$ be defined by $u'(i') = u'(i, \eta) = 0$ if $\eta \not = \eta^\star$ and $u'(i, \eta) = u(i)$ otherwise. Then
\begin{align*}
    \sumK \sum_{\eta \in H_T} \frac{1}{\eta}u'(i, \eta) \ln\left(\frac{u'(i, \eta)} {p_1^{\Psi}(i, \eta)}\right)
    & = \frac{1}{\eta^\star} \sumK u(i) \ln\left(\frac{u(i)} {\pi(i) \frac{2^{-2j^\star}}{\sum_{j' = 1}^J 2^{-2j'}}}\right) \\
    & = \frac{1}{\eta^\star} \left(\sumK u(i) \ln\left(\frac{u(i)} {\pi(i)}\right) + \sumK u(i) \ln\left( \frac{\sum_{j' = 1}^J 2^{-2j'}}{2^{-2j^\star}}\right)\right) \\
    & = \frac{1}{\eta^{\star}} \left(\KL(\u, \pi) + 2j^\star  + \ln\left( \sum_{j' = 1}^J 2^{-2j'}\right)\right) \\
    & \leq \frac{\KL(\u, \pi) + J}{\eta^{\star}}
\end{align*}
where we used that $\sum_{j = 1}^J 2^{-2j} = \frac{1}{3} - \frac{4^{-J}}{3} \leq 1$. We now bound the final sum on the right-hand side of equation \eqref{eq:expanded R_T fullinfo} by using that $\frac{1}{\min\{a, b\}} \leq \frac{1}{a} + \frac{1}{b}$ for $a, b > 0$:
\begin{equation}\label{eq:boundingprior}
\begin{split}
    \sumK \sum_{\eta \in H_T} \frac{1}{\eta} p_1^\Psi(i, \eta)
    & = \sumK \sum_{j=1}^{J} \frac{ \frac{2^{-2j}}{\sum_{j' \in [J]}2^{-2j'}}\pi(i)}{\min\bigg\{ \frac{1}{4(1 + \dtmax)},  \frac{\sqrt{\ln(K) + 2(\ln(T) + 1)}}{4\sqrt{1+\dtmax}}2^{-j}\bigg\}} \\
    & \leq 4(1 + \dTmax) + \frac{4\sqrt{1+\dTmax}}{\sqrt{\ln(K) + 2(\ln(T) + 1)}} \frac{\sum_{j = 1}^J 2^{-j}}{\sum_{j = 1}^J 2^{-2j}} \\
    & =  4(1 + \dTmax) + \frac{12\sqrt{1+\dTmax}}{\sqrt{\ln(K) + 2(\ln(T) + 1)}}  \frac{1-2^{-J}}{1-4^{-J}}  \\
    & \leq 4(1 + \dTmax) + 12\sqrt{1+\dTmax}
\end{split}
\end{equation}
where in the last equality we used again that $\sum_{j = 1}^J 2^{-2j} = \frac{1}{3} - \frac{4^{-J}}{3}$ and that $\sum_{j = 1}^J 2^{-j} = 1 - 2^{-J}$. 
We continue by further bounding the drift term using the fact that $\nabla^2 R_t(\p_t^{av})$ is a diagonal matrix with diagonal
\[
    \big(\nabla^2 R_t(\p_t^{av})\big)_{i'i'} = \frac{1}{\eta p_t^{av}(i, \eta)}
\]
for $i' = (i,\eta)$. We have
\begin{align*}
    (\bLhat_t - \bLhat_t^{av})^\top \big(\nabla^{2} R_t(\p_t^{av})\big)^{-1} \blhat_t 
    & = \sumK \sum_{\eta \in H_t} \big(\Lhat_t(i, \eta) - \Lhat_t^{av}(i, \eta)\big)\eta p_t^{av}(i, \eta) \lhat_t(i, \eta) \\
    & \leq \sumK \sum_{\eta \in H_t} \big(1 + \rho_t(i, \eta)\big) p_t^{av}(i, \eta) \ltil_t(i, \eta) = \inner{\p_t^{av}}{\a_t}
\end{align*}
where in the last step we used~\eqref{eq:key}, the inequality $\lhat_t(i') \leq 2 \ltil_t(i') = 2\ell_t(i)$, and the assumption $\ell_t(i) \in [0, 1]$.

Combining the above and using that $\sumT \inner{\u'}{\a_t} = \eta^\star \sumK  u(i) \sumT 4 \ell_t(i)(1 + \rho_t(i)) = 4\eta^\star \inner{\u}{\L_T + \L_T^\rho}$ and $J \le \ln(T) + 1$, we continue from~\eqref{eq:thfullinfo1}: 
\begin{equation}\label{eq:fulltooptimize}
\begin{split}
    \sumT & \inner{\q_t - \u}{\bell_t}
    \leq 4(1 + \dTmax) + 12\sqrt{1+\dTmax} + \frac{\KL(\u, \pi) + 2(\ln(T) + 1)}{\eta^\star} + 4\eta^\star\inner{\u}{\L_T + \L_T^\rho} 
\end{split}
\end{equation}
Ideally, we would set $\eta^{\star}$ to
\begin{equation}
\label{eq:etaopt}
    \sqrt{\frac{\KL(\u, \pi) + 2(\ln(T) + 1)}{4\inner{\u}{\L_T + \L_T^\rho}}}
\end{equation}
However, we can only pick a learning rate from $H_T$. We now show that $H_T$ contains a good approximation of the quantity in~\eqref{eq:etaopt}. We split the remainder of the analysis into two cases. First, assume
\[
    \max_{\eta \in H_T} \eta \leq \sqrt{\frac{\KL(\u, \pi) + 2(\ln(T) + 1)}{4\inner{\u}{\L_T + \L_T^\rho}}}
\]
and thus
\begin{align*}
    4\inner{\u}{\L_T + \L_T^\rho} \leq \frac{\KL(\u, \pi) + 2(\ln(T) + 1)}{(\max_{\eta \in H_T} \eta)^2}
\end{align*}
Using the above in equation \eqref{eq:fulltooptimize} and choosing $\eta^\star = \max_{\eta \in H_T} \eta$ we obtain
\begin{equation*}
\begin{split}
    \sumT \inner{\q_t - \u}{\bell_t} 
    & \leq 4(1 + \dTmax) + 12\sqrt{1+\dTmax}   + \frac{2(\KL(\u, \pi) + 2(\ln(T) + 1))}{\eta^\star} \\
    & \leq 4(1 + \dTmax) + 12\sqrt{1+\dTmax} + 8(\KL(\u, \pi) + 2(\ln(T) + 1))(1 + \dTmax) \\
    & + 16\sqrt{(1 + \dTmax)(\ln(K) + 2(1 + \ln(T)))}
\end{split}
\end{equation*}
where in the last step we used
\[
    \frac{1}{\displaystyle{\max_{\eta \in H_T}\eta}} \le 4(1 + \dTmax) + 8\sqrt{\frac{1+\dTmax}{\ln(K) + 2(\ln(T) + 1)}}
\]
In the second case
\[
\max_{\eta \in H_T} \eta > \sqrt{\frac{\KL(\u, \pi) + 2(\ln(T) + 1)}{4\inner{\u}{\L_T + \L_T^\rho}}}
\]
This implies that there is an $\eta \in H_T$ that is within a factor 2 of $\sqrt{\frac{\KL(\u, \pi) + 2(\ln(T) + 1)}{\inner{\u}{\L_T + \L_T^\rho}}}$ as 
\begin{align*}
\min_{\eta \in H_T} \eta   & \leq \sqrt{\frac{\KL(\u, \pi) + 2(\ln(T) + 1)}{4(1 + \dTmax)T}} \leq \sqrt{\frac{\KL(\u, \pi) + 2(\ln(T) + 1)}{4\inner{\u}{\L_T + \L_T^\rho}}} \leq \max_{\eta \in H_T} \eta
\end{align*}
and thus 
\begin{align*}
    & \frac{\KL(\u, \pi) + 2(\ln(T) + 1)}{\eta^\star} + \eta^\star 4 \inner{\u}{\L_T + \L_T^\rho} \leq 8 \sqrt{(\KL(\u, \pi) + 2(\ln(T) + 1))\inner{\u}{\L_T + \L_T^\rho}}.
\end{align*}
Therefore, we have that in the second case 
\begin{align*}
    \sumT \inner{\q_t - \u}{\bell_t} \leq 4(1 + \dTmax) + 12\sqrt{1+\dTmax} + 8 \sqrt{(\KL(\u, \pi) + 2(\ln(T) + 1))\inner{\u}{\L_T + \L_T^\rho}}.
\end{align*}
Combining the first and second case completes the proof. 
\end{proof}

\section{DETAILS OF SECTION \ref{sec:partcon}}\label{app:partcon}

\banditFbound*
\begin{proof}
We first show that $F_t$ is Legendre.
Note that $F_t$ is strictly convex because $\Phi_t$ is strictly convex and $\Psi_t$ is convex.
% Since $\nabla^2 F_t = \nabla^2 \Psi_t + \nabla^2 \Phi_t$ is positive definite because $\nabla^2 \Phi_t$ is positive definite and $\Psi_t$ is convex, $F_t$ is strictly convex.
Now pick any sequence $\x_1, \x_2, \ldots$ in the interior of the domain of $F_t$ converging to a boundary point.
We have that $\|\nabla F_t(\x_n)\| \to \infty$, because each component
\begin{align*}
    \big(\nabla F_t(\x)\big)_{(i,\gamma)} = -\frac{1}{\eta_t\sumgam x(i, \gamma)} + \frac{1}{\gamma}\ln\big(x(i, \gamma)\big) + \frac{1}{\gamma}
\end{align*}
of the gradient of $F_t$ is increasing in each coordinate of $\x$. The same observation, together with the fact that both $\inner{\nabla \Psi_t(\cdot)}{\z}$ and $\inner{\nabla \Phi_t(\cdot)}{\z}$ are concave, shows that $W_t(\cdot, \z) = \inner{\nabla F_t(\cdot)}{\z}$ satisfies the conditions of Lemma~\ref{lem:drift}.
% observe that both $\inner{\nabla \Psi_t(\cdot)}{\z}$ and $\inner{\nabla \Phi_t(\cdot)}{\z}$ are concave and that $(\nabla F_t(\cdot))_{m}$ is increasing in the corresponding coordinate of $\x$ for all $m$. 

Observe that $\nabla^2 F_t(\p_t^{av})$ is a block-diagonal matrix with blocks $i = 1, \ldots, K$
\begin{equation*}
\begin{split}
    B_t(i) =
     & \frac{\1 \1^\top}{\eta_t \left(q_t(i)\right)^2} + \diag\left(\frac{1}{\gamma_1 p_t^{av}(i, \gamma_1)}, \ldots, \frac{1}{\gamma_J p_t^{av}(i, \gamma_J)}\right)
\end{split}
\end{equation*}
of size $J \times J$ each.
Denote by $V_t = \diag(\v_t)$, where $\v_t = \big({\gamma_1 p_t^{av}(i, \gamma_1)}, \ldots, {\gamma_J p_t^{av}(i, \gamma_J)}\big)$. The inverse of $B_t(i)$ can be computed by employing the Sherman-Morrison formula:
\begin{equation}\label{eq:inverseblock}
\begin{split}
    B_t(i)^{-1} & = V_t - \frac{\eta_t^{-1} q_t(i)^{-2} V_t\1\1^{\top}V_t }{1 + \eta_t^{-1} q_t(i)^{-2}\inner{\v_t}{\1}} =  V_t - \frac{\v_t \v_t^{\top}}{\eta_t q_t(i)^{2} + \inner{\v_t}{\1}}
\end{split}
\end{equation}
Note that $\blhat_t$ is only non-zero in the $J$ coordinates of the form $(i_t,\gamma)$ for $\gamma\in\Gamma$. Let $\h$ be the $J$-vector including only these non-zero elements of $\blhat_t$, and let $\b = q_t(i_t)\h$ so that $b(j) = \ell_t(i_t) + q_t(i_t) a_t(i_t,\gamma_j) = \ell_t(i_t) +  4\gamma_j\ell_t(i_t)\rho_t(i)$. Denote by $\0_{K \times (J-1)}$ a vector of zeros of length $K \times (J-1)$. Since the block-diagonal structure is preserved when taking inverses, $\big(\nabla^{2} F_t(\p_t^{av})\big)^{-1}$ is a block-diagonal matrix with blocks $B_t(i)^{-1}$ and
$
    \big(\nabla^{2} F_t(\p_t^{av})\big)^{-1} \blhat_t = \big(\0_{K \times (J-1)}, B_t(i_t)^{-1}\h\big)
$.
Next, we write $\b = Y \1 \ell_t(i_t)$, where
\[
    Y = \I_J + 4\rho_t(i)\diag( \gamma_1, \ldots, \gamma_J)
\]
and $\I_J$ is the $J \times J$ identity matrix. Using~\eqref{eq:inverseblock} and $Y \preceq 2 \I_J$ because $\max_j\gamma_j \le (4\rho^{\star})^{-1}$, we continue with 
\begin{equation}\label{eq:banditvariance1}
\begin{split}
    \blhat_t^\top \big(\nabla^{2} F_t(\p_t^{av})\big)^{-1} \blhat_t &=
    \frac{\ell_t(i_t)^2}{q_t(i_t)^2} \1^\top Y B_t(i_t)^{-1} Y \1
\\& \le
    \frac{4\ell_t(i_t)^2}{q_t(i_t)^2} \1^\top  B_t(i_t)^{-1} \1
%\\&=
%    \frac{4\ell_t(i_t)^2}{q_t(i_t)^2}\left(\1^\top V_t\1 - \frac{\inner{\v_t}{\1}^2}{\eta_t q(i_t)^2 + \1^\top V_t \1}\right)
\\&=
    \frac{4\ell_t(i_t)^2}{q_t(i_t)^2}\left(\inner{\v_t}{\1} - \frac{\inner{\v_t}{\1}^2}{\eta_t q(i_t)^2 + \inner{\v_t}{\1}}\right)
\\&=
    \frac{4\ell_t(i_t)^2}{q_t(i_t)^2}\frac{\eta_t q(i_t)^2\inner{\v_t}{\1}}{\eta_t q(i_t)^2 + \inner{\v_t}{\1}}
\\&\le
    4\eta_t\ell_t(i_t)^2
\end{split}
\end{equation}
where in the last step we used $\v_t \ge 0$.
Denote by
$
    \L^{\textnormal{miss}} = \bLhat_t - \blhat_t - \bLhat_t^{av}
$
and denote by $\y$ the components of $\L^{\textnormal{miss}}$ corresponding to $i_t$.
Namely,
\[
    y(i_t,\gamma) = \sum_{s \in m_t(i_t)} \lhat_t(i_t, \gamma)
\]
Since $\inner{\y}{\v_t}\inner{\v_t}{\h} \geq 0$ (these are all nonnegative vectors), we have that 
\begin{equation}\label{eq:banditdelay1}
\begin{split}
    \L^{\textnormal{miss}} \big(\nabla^{2} F_t(\p_t^{av})\big)^{-1} \blhat_t
&=
    \y^{\top} B_t(i_t)^{-1} \h
\\&=
    \y^{\top} V_t \h - \frac{\inner{\y}{\v_t}\inner{\v_t}{\h}}{\eta_t q(i_t)^2 + \inner{\v_t}{\1}}
\\&\le
    \y^{\top} V_t \h
\\&=
    \sum_{\gamma \in \Gamma} \sum_{s \in m_t(i_t)} p_t^{av}(i_t, \gamma) \gamma \lhat_t(i_t, \gamma) \lhat_s(i_t, \gamma)
\\&\le
    4 \sum_{\gamma \in \Gamma} \sum_{s \in m_t(i_t)} p_t^{av}(i, \gamma) \gamma \ltil_t(i_t, \gamma) \ltil_s(i_t, \gamma)
\end{split}
\end{equation}
where in the last step we used $\lhat_t(i') \leq 2 \ltil_t(i')$.
Combining equations \eqref{eq:banditvariance1} and \eqref{eq:banditdelay1} we obtain 
\begin{align*}
     (\bLhat_t & - \bLhat_t^{av})^\top \big(\nabla^{2} F_t(\p_t^{av})\big)^{-1} \blhat_t \leq \eta_t 4\ell_t(i_t)^2  + 4 \sum_{\gamma \in \Gamma} \sum_{s \in m_t(i_t)} p_t^{av}(i_t, \gamma) \gamma \ltil_t(i_t, \gamma) \ltil_s(i_t, \gamma)
\end{align*}
which completes the proof.
\end{proof}
%%%%%%%%%%%%%%%%%%%%%%%%%%%%%%%%%%%%%%%%%%%%%%%%%%%%%%%%%%%%%%%%%%%%%%%
\begin{restatable}{relemma}{banditRbound}
\label{lem:banditRbound}
Let $\p_1^{\Psi} \equiv \frac{1}{KJ}$, let $\p_1^\Phi$ be set as in~\eqref{eq:fullinfoprior} and $\Gamma$ be set as in~\eqref{eq:banditgamma}. If $\p^\star = (1 - \alpha)\u' + \alpha\frac{\1}{K'}$ with $\alpha = \frac{1}{T}$ and $\pi_1(i) \geq \frac{1}{T^2}$, then $R_t(\p) \geq R_{t-1}(\p)$ for all $t \geq 1$ and $\p \in \triangle'$. Furthermore 
\begin{align*}
     \E&\Bigg[\sumT\inner{\p^\star - \u'}{\blhat_t}\Bigg] + R_T(\p^\star) \leq 50 \dstar
    + \frac{K \ln(T)}{\eta_T} + \frac{1 + \ln(T) + \KL(\u, \pi)}{\gamma^\star}
\end{align*}
\end{restatable}
\begin{proof}
Let $\phi_\gamma(x) = \frac{x}{\gamma}\ln(x)$ and $\psi(x) = - \ln(x)$. To see that $R_t(\p) \geq R_{t-1}(\p)$ observe that 
\begin{align*}
    R_t(\p) & = B_{\Psi_t}(\p, \p_1^{\Psi}) + B_{\Phi_t}(\p, \p_1^{\Phi}) \\
    & = \sumK \frac{1}{\eta_t} B_{\psi}\left(\sumgam p(i, \gamma), \sumgam p_1^{\Psi}(i, \gamma)\right) \\
    & + \sumK \sumgam B_{\phi_{\gamma}}\big(p(i, \gamma), p_1^{\Phi}(i, \gamma)\big) \\ 
    & \geq \sumK \frac{1}{\eta_{t-1}} B_{\psi}\left(\sumgam p(i, \gamma), \sumgam p_1^{\Psi}(i, \gamma)\right) \\
    & + \sumK \sumgam B_{\phi_{\gamma}}\big(p(i, \gamma), p_1^{\Phi}(i, \gamma)\big) \\
    & = B_{\Psi_{t-1}}(\p, \p_1^{\Psi}) + B_{\Phi_{t-1}}(\p, \p_1^{\Phi}) = R_{t-1}(\p)
\end{align*}
where the inequality is due to the non-negativity of the Bregman divergence and the fact that $\eta_t \leq \eta_{t-1}$.

Now
$
    a_t(i') = 4 \ltil_t(i') \gamma(i') \rho_t  \leq  \ltil_t(i) 
$
because $\max_j\gamma_j \leq \big(4 \dstar\big)^{-1}$,
and thus
\[
    \sumT\Eb{\inner{\p^\star - \u'}{\blhat_t}}
\le
    \sumT\Eb{\inner{\frac{\alpha}{K'}\1}{\blhat_t}}
\le
    2 \alpha T
\]
For any $\p^\star$ we have that:
\begin{equation}\label{eq:RT1}
\begin{split}
    & R_T(\p^\star)
    =  -\frac{1}{\eta_T}\sumK \ln\left(K\sumgam p^\star(i, \gamma)\right)  + \sumK \sumgam \frac{1}{\gamma} \left(p^\star(i, \gamma)\ln\left(\frac{p^\star(i, \gamma)}{p_1^\Phi(i, \gamma)}\right) + p_1^\Phi(i, \gamma) - p^\star(i, \gamma)\right)
\end{split}
\end{equation}
We continue by bounding the first sum on the right hand side of equation \eqref{eq:RT1}: 
\begin{align*}
    -\sumK\ln\left(K\sumgam p^\star(i, \gamma)\right)
&= 
    -\sumK \ln\left(\sumgam (1-\alpha) u'(i, \gamma) K + \alpha\right)
\\&\le
    - K \ln(\alpha)
\end{align*}
Next, we bound the second sum on the right hand side of~\eqref{eq:RT1}: 
\begin{align*}
    \sumK &\sumgam \frac{1}{\gamma} \left(p^\star(i, \gamma)\ln\left(\frac{p^\star(i, \gamma)}{p_1^\Phi(i, \gamma)}\right) + p_1^\Phi(i, \gamma) - p^\star(i, \gamma)\right)
\\&\le
    \sumK \sumgam \frac{1}{\gamma} \left(p^\star(i, \gamma)\ln\left(\frac{p^\star(i, \gamma)}{p_1^\Phi(i, \gamma)}\right) + p_1^\Phi(i, \gamma)\right)
\\&\le
    \sumK \sumgam \frac{1}{\gamma} p^\star(i, \gamma)\ln\left(\frac{p^\star(i, \gamma)}{p_1^\Phi(i, \gamma)}\right) + 4\dstar
    + \frac{12\sqrt{\dstar}}{\sqrt{\ln(K) + \ln(T) + 1}}
\\&\le
    \sumK \sumgam\frac{1}{\gamma} p^\star(i, \gamma)\ln\left(\frac{p^\star(i, \gamma)}{p_1^\Phi(i, \gamma)}\right) + 4\dstar + 12\sqrt{\dstar}
\end{align*}
where we used that $\sum_{j = 1}^{J}2^{-2j} = 1/3 - \frac{4^{-J}}{3}$ and that that $\frac{1}{\min\{a, b\}} \leq \frac{1}{a} + \frac{1}{b}$ for $a, b > 0$, see also~\eqref{eq:boundingprior}.
Now, to bound the double sum on the right-hand side of the last inequality we use Jensen's inequality, which we may use due the convexity of $f(x) = x\ln(xa)$ for $a > 0$:
%
% \begin{align*}
%     & \sumK \sumgam \frac{1}{\gamma} p^\star(i, \gamma)\ln\left(\frac{p^\star(i, \gamma)}{p_1^\Phi(i, \gamma)}\right) \\
%     & \leq  \sumK \sumgam \frac{1}{\gamma} u'(i, \gamma)\ln\left(\frac{u'(i, \gamma)}{p_1^\Phi(i, \gamma)}\right) \\
%     & + \sumK \sumgam\frac{\frac{\alpha}{K'} - \alpha u'(i, \gamma)}{\gamma}(1 + \ln (p^\star(i, \gamma)) - \ln(p_1^\Phi(i, \gamma))) \\
%     & \leq  \sumK \sumgam \frac{1}{\gamma} u'(i, \gamma)\ln\left(\frac{u'(i, \gamma)}{p_1^\Phi(i, \gamma)}\right) \\
%     & + 2K'\alpha \max_{i, \gamma}\bigg\{\frac{1}{\gamma}(1 -\ln(\alpha/K') - \ln(p_1^\Phi(i, \gamma))\bigg\} %\\
%     %& \leq \sumK \sumgam \frac{1}{\gamma} u'(i, \gamma)\ln\left(\frac{u'(i, \gamma)}{p_1^\Phi(i, \gamma)}\right) %\\
%     %& + 2K'\alpha \max_{i, j}\bigg\{\frac{1}{\gamma_j}(1 -\ln(\alpha/K') + 2J - \ln(\pi(i)))\bigg\}\\
% \end{align*}
%
\begin{align*}
    \sumK \sumgam \frac{1}{\gamma} p^\star(i, \gamma)\ln\left(\frac{p^\star(i, \gamma)}{p_1^\Phi(i, \gamma)}\right)
& \le
    \sumK \sumgam \frac{1-\alpha}{\gamma} u'(i, \gamma)\ln\left(\frac{u'(i, \gamma)}{p_1^\Phi(i, \gamma)}\right)
    + \sumK \sumgam\frac{{\alpha}/{K'}}{\gamma}\ln\left(\frac{1/K'}{p_1^\Phi(i,\gamma)}\right)
\\&\le
    \sumK \sumgam \frac{1-\alpha}{\gamma} u'(i, \gamma)\ln\left(\frac{u'(i, \gamma)}{p_1^\Phi(i, \gamma)}\right)
    + \alpha \max_{i, \gamma}\left\{\frac{1}{\gamma}\left(-\ln\big(p_1^\Phi(i, \gamma)\big)-\ln(K') \right)\right\}
\end{align*}
Since
\begin{align*}
    -\ln\big(p_1^{\Phi}(i, \gamma)\big) \leq & -\ln(\pi(i)) + 2J +  \ln\left(\sum_{j = 1}^{J} 2^{-2j}\right)\leq -\ln(\pi(i)) + 2J
\end{align*}
and $J \leq \ln(T) + 1$, $1/\gamma \leq 4 \dstar \sqrt{T}$, and $\pi_1(i) \geq \frac{1}{T^2}$ we have that 
\begin{align*}
    & \alpha \max_{i, \gamma}\left\{\frac{1}{\gamma}\left(-\ln\big(p_1^\Phi(i, \gamma)\big)-\ln(K') \right)\right\} \leq  16 \alpha \dstar \sqrt{T}(1 + \ln(T))  \leq 32 \alpha \dstar T
\end{align*}
where we used that $\ln(x) \leq \sqrt{x}$.
Now, setting $u'(i, \gamma) = u(i)$ if $\gamma = \gamma^\star$ and 0 otherwise, we obtain
\begin{align*}
    \sumK \sumgam \frac{1-\alpha}{\gamma} u'(i, \gamma)\ln\left(\frac{u'(i, \gamma)}{p_1^\Phi(i, \gamma)}\right) 
    & =  \frac{1-\alpha}{\gamma^\star}\left(-\ln\left(\frac{2^{-2j}}{\sum_{j \in [J]} 2^{-2J}}\right) + \sumK u(i)\ln\left(\frac{u(i)}{\pi(i)}\right)\right) \\
    & \leq  \frac{1}{\gamma^\star}\left(2J + \KL(\u, \pi)\right) \\
    & \leq  \frac{1}{\gamma^\star}\left(1 + \ln(T) + \KL(\u, \pi)\right)
\end{align*}
Combining the above and setting $\alpha = \frac{1}{T}$ we find
\begin{align*}
     \Eb{\sumT\inner{\p^\star - \u'}{\blhat_t}} + R_T(\p^\star) 
    & \leq  32 \alpha \dstar T  + \Eb{- \frac{1}{\eta_T}K \ln(\alpha)} + 2 \alpha T + 4\dstar + 12\sqrt{\dstar}\\
    & \quad + \frac{1}{\gamma^\star}\left(1 + \ln(T) + \KL(\u, \pi)\right)\\
    & \leq 50 \dstar + \Eb{\frac{K \ln(T)}{\eta_T}} + \frac{1}{\gamma^\star}\left(1 + \ln(T) + \KL(\u, \pi)\right)
\end{align*}
which completes the proof.
\end{proof}
\begin{theorem}\label{th:banditbound}
Let $\p_1^{\Psi} \equiv \frac{1}{KJ}$ and let $\p_1^\Phi$ be given by~\eqref{eq:fullinfoprior}. Assume $\pi_1(i) \geq \frac{1}{T^2}$ for all $i\in [K]$.
If we run Algorithm~\ref{alg:framework} with regularizers~\eqref{eq:banditreg1} and~\eqref{eq:banditreg2}, corresponding learning rates~\eqref{eq:banditgamma} and~\eqref{eq:banditeta}, then the predictions $i_t \sim \q_t$, with $\q_t$ as in \eqref{eq:banditpredictions1}, satisfy
% Let $p_1^{\Psi}(i, \gamma) = \frac{1}{KJ}$, let $\p_1^\Phi$ be as specified in \ref{eq:fullinfoprior}, and suppose that $\pi_1(i) \geq \frac{1}{T^2}$. With $\p_t^{av}$ from Algorithm \ref{alg:framework} with regularizers \eqref{eq:banditreg1} and \eqref{eq:banditreg2} and corresponding learning rates in equations \eqref{eq:banditgamma} and \eqref{eq:banditeta}, the predictions $i_t \sim \q_t$, with $\q_t$ as in \eqref{eq:banditpredictions1}, satisfy  
%
\begin{align*}
    \E\big[\sumT (\ell_t(i_t) - \inner{\u}{\bell_t})\big]
& \le  12 \sqrt{K\ln(T)\inner{\u}{\L_T}}  + 16 \sqrt{\big(\KL(\u, \pi) + \ln(T) + 1\big)\inner{\u}{\L_T^\rho}}  \\ 
    & \quad + 48\big(5 + \ln(T) + \KL(\u, \pi)\big)\dstar + 42K\ln(T)
\end{align*}
\end{theorem}
\begin{proof}
Let $\p^\star = (1 - \alpha)\u' + \alpha\1\frac{1}{K'}$, where $\alpha = \frac{1}{T}$. We start by applying Lemma \ref{lem:frameworkregret} to bound the expected regret against $\p^\star$, which we may do because the conditions of Lemma \ref{lem:frameworkregret} are satisfied as per Lemma \ref{lem:banditFbound}:
\begin{align*}
    \E\Bigg[\sumT   \inner{\q_t - \u}{\bell_t}\bigg] & \leq  \sumT \Eb{\inner{\u}{\a_t}} + \E\Bigg[\sumT \inner{\p^\star - \u'}{\lhat_t}\bigg] +  \Eb{R_T(\p^\star)}\\
    & \quad + \Eb{\sumT \left(\big(\bLhat_t - \bLhat_t^{av}\big)^\top \big(\nabla^{2} R_t(\p_t^{av})\big)^{-1} \blhat_t - \inner{\p_t^{av}}{\a_t}\right)}
\end{align*}
Let $u'(i') = u'(i, \gamma) = u(i)$ if $\gamma = \gamma^\star$ and 0 otherwise and recall that $L_T^d(i) = \sumT \ell_t(i)(1 + \rho_t(i))$. Since $\E[\ell_t(i)\id[i = i_t]q_t(i)^{-1}] = \ell_t(i)$ we have that $\sumT \Eb{\inner{\u}{\a_t}} = \sumT 4 \gamma^\star \inner{\u}{\L_T^\rho}$. By applying Lemmas \ref{lem:banditRbound} and \ref{lem:banditdriftbound} we may further bound the expected regret:
\begin{equation}\label{eq:bandittoopt}
\begin{split}
    \E\Bigg[\sumT   \inner{\q_t - \u}{\bell_t}\bigg] \leq & 50 \dstar
    + \Eb{\sumT \eta_t 4\ell_t(i_t)^2} + \Eb{\frac{1}{\eta_T}K \ln(T)} \\
    &  + \sumT 4 \gamma^\star \inner{\u}{\L_T^\rho} + \frac{1}{\gamma^\star}\big(1 + \ln(T) + \KL(\u, \pi)\big)
\end{split}
\end{equation}
Ideally, we would have that $\gamma^\star = \sqrt{\frac{\KL(\u, \pi) + \ln(T) + 1}{4\inner{\u}{\L_T^\rho}}}$. However, we have a grid of learning rates $\Gamma$ from which we have to choose an appropriate learning rate. Instead, we will show that $\Gamma$ contains a close approximation of the optimal learning rate. As in the proof of Theorem \ref{th:fullinfobound} we will split the analysis into two case. In the first case $\max_{\gamma \in \Gamma} \gamma \leq \sqrt{\frac{\KL(\u, \pi) + \ln(T) + 1}{4\inner{\u}{\L_T + \L_T^\rho}}}$ and thus
\begin{align*}
    4\inner{\u}{\L_T^\rho} \leq \frac{\KL(\u, \pi) + 2(\ln(T) + 1)}{(\max_{\gamma \in \Gamma} \gamma)^2}
\end{align*}
Thus, choosing $\gamma^\star = \max_{\gamma \in \Gamma} \gamma$ we obtain
\begin{equation*}
\begin{split}
    4 \gamma^\star \inner{\u}{\L_T^\rho} + \frac{1}{\gamma^\star}\big(1 + \ln(T) + \KL(\u, \pi)\big) 
    & \leq  \frac{2}{\gamma^\star}\big(1 + \ln(T) + \KL(\u, \pi)\big) \\
    & \leq 8\big(1 + \ln(T) + \KL(\u, \pi)\big)\big(\dstar + 2\sqrt{\dstar}\big)
\end{split}
\end{equation*}
In the second case $\max_{\gamma \in \Gamma} \gamma > \sqrt{\frac{\KL(\u, \pi) + \ln(T) + 1}{4\inner{\u}{\L_T^\rho}}}$. This implies that there is an $\gamma \in \Gamma$ that is within a factor 2 of $\sqrt{\frac{\KL(\u, \pi) + \ln(T) + 1}{\inner{\u}{\L_T^\rho}}}$ as 
\begin{align*}
& \min_{\gamma \in \Gamma} \gamma \leq \sqrt{\frac{\KL(\u, \pi) + \ln(T) + 1}{4(1 + \dTmax)T}}  \leq \sqrt{\frac{\KL(\u, \pi) + \ln(T) + 1}{4\inner{\u}{\L_T^\rho}}} \leq \max_{\gamma \in \Gamma} \gamma 
\end{align*}
and thus 
\begin{align*}
    & 4 \gamma^\star \inner{\u}{\L_T^\rho} + \frac{1}{\gamma^\star}\big( + \ln(T) + \KL(\u, \pi)\big) \leq 8 \sqrt{\big(\KL(\u, \pi) + \ln(T) + 1\big)\inner{\u}{\L_T^\rho}}
\end{align*}
Therefore, combining the two cases we find that 
\begin{equation}\label{eq:banditgammaopt}
\begin{split}
    4 \gamma^\star \inner{\u}{\L_T^\rho} + \frac{1}{\gamma^\star}\left(1 + \ln(T) + \KL(\u, \pi)\right) 
    \leq & 8 \sqrt{(\KL(\u, \pi) + \ln(T) + 1)\inner{\u}{\L_T^\rho}}  \\
    & + 8\left(1 + \ln(T) + \KL(\u, \pi)\right)\left(\dstar + 2\sqrt{\dstar}\right)
\end{split}
\end{equation}
As for the terms involving $\eta_t$, we have that 
\begin{align*}
    \eta_t 4 \ell_t(i_t) & = 4\ell_t(i_t)\sqrt{\frac{K\ln(T)}{4(1 + \dstar) + 4\sum_{s \in \Sset_t} \ell_s(i_s)}} \leq 4\ell_t(i_t)\sqrt{\frac{K\ln(T)}{4\sum_{s \leq t} \ell_s(i_s)}}
\end{align*}
Summing over $t$ and using that $\sum_{t = 1}^t \frac{x_t}{\sum_{s \leq t} x_t} \leq 2 \sqrt{\sumT x_t}$ we find that 
\begin{align*}
    \sumT \eta_t 4 \ell_t(i_t) \leq 4 \sqrt{K\ln(T)\sumT \ell_t(i_t)}
\end{align*}
which means that 
\begin{equation}\label{eq:banditetaopt}
\begin{split}
    \Eb{\sumT \eta_t 4\ell_t(i_t)^2 + \frac{1}{\eta_T}K \ln(T)} 
    & \leq 6 \Eb{\sqrt{K\ln(T)(1 + \dstar) + K\ln(T) \sumT \ell_t(i_t)}}\\
    & \leq 6 \Eb{\sqrt{K\ln(T) \sumT \ell_t(i_t)}} + 6\sqrt{K\ln(T)(1 + \dstar)}\\
    & \leq 6 \sqrt{K\ln(T) \sumT \Eb{\ell_t(i_t)}} + 6\sqrt{K\ln(T)(1 + \dstar)}
\end{split}
\end{equation}
where we used that $\sqrt{x + y} \leq \sqrt{x} + \sqrt{y}$ for $x, y \geq 0$ and Jensen's inequality. Combining equations \eqref{eq:bandittoopt}, \eqref{eq:banditgammaopt}, and \eqref{eq:banditetaopt} we obtain
\begin{align*}
    \E\Bigg[\sumT   \inner{\q_t - \u}{\bell_t}\bigg] \leq &  50 \dstar
    + 6 \sqrt{K\ln(T) \sumT \Eb{\ell_t(i_t)}} + 6\sqrt{K\ln(T)(1 + \dstar)} + 8 \sqrt{(\KL(\u, \pi) + \ln(T) + 1)\inner{\u}{\L_T^\rho}}  \\ 
    & + 8\left(1 + \ln(T) + \KL(\u, \pi)\right)\left(\dstar + 2\sqrt{\dstar}\right)\\
    \leq  & 6 \sqrt{K\ln(T) \sumT \Eb{\ell_t(i_t)}} + 6\sqrt{K\ln(T)(1 + \dstar)} + 8 \sqrt{(\KL(\u, \pi) + \ln(T) + 1)\inner{\u}{\L_T^\rho}}  \\ 
    & + 24\big(4 + \ln(T) + \KL(\u, \pi)\big)\dstar
\end{align*}
To complete the proof observe that $\Eb{\ell_t(i_t)} = \Eb{\inner{\q_t}{\bell_t}}= \Eb{\inner{\q_t - \u}{\bell_t}} + \inner{\u}{\bell_t}$. Using that $\sqrt{x + y} \leq \sqrt{x} + \sqrt{y}$ and using that $\sqrt{xy} = \half \inf_{\zeta > 0} \frac{x}{\zeta} + \zeta y$ for $x, y \geq 0$ we find
\begin{align*}
    \E\Bigg[\sumT   \inner{\q_t - \u}{\bell_t}\bigg] \leq 
    & 6 \sqrt{K\ln(T)\left(\E\Bigg[\sumT   \inner{\q_t - \u}{\bell_t} \bigg] + \inner{\u}{\L_T}\right)} 
    + 6\sqrt{K\ln(T)(1 + \dstar)} \\
    & + 8 \sqrt{(\KL(\u, \pi) + \ln(T) + 1)\inner{\u}{\L_T^\rho}}  + 24\left(4 + \ln(T) + \KL(\u, \pi)\right)\dstar \\
    \leq & 6 \sqrt{K\ln(T)\inner{\u}{\L_T}} + 6\sqrt{K\ln(T)(1 + \dstar)} + 8 \sqrt{(\KL(\u, \pi) + \ln(T) + 1)\inner{\u}{\L_T^\rho}}  \\ 
    & + 24\left(4 + \ln(T) + \KL(\u, \pi)\right)\dstar  + \half \E\Bigg[\sumT   \inner{\q_t - \u}{\bell_t} \bigg] + 18 K\ln(T)\\
    \leq & 6 \sqrt{K\ln(T)\inner{\u}{\L_T}} + 8 \sqrt{(\KL(\u, \pi) + \ln(T) + 1)\inner{\u}{\L_T^\rho}}  \\ 
    & + 24\left(5 + \ln(T) + \KL(\u, \pi)\right)\dstar + \half \E\Bigg[\sumT   \inner{\q_t - \u}{\bell_t} \bigg] + 21 K\ln(T)
\end{align*}
After reordering the above equation gives us
\begin{align*}
    \half \E\Bigg[\sumT   \inner{\q_t - \u}{\bell_t}\bigg] \leq & 21K\ln(T) + 6 \sqrt{K\ln(T)\inner{\u}{\L_T}} \\
    & + 8 \sqrt{(\KL(\u, \pi) + \ln(T) + 1)\inner{\u}{\L_T^\rho}}  + 24\left(5 + \ln(T) + \KL(\u, \pi)\right)\dstar
\end{align*}
which completes the proof after multiplying both sides by 2
\end{proof}

\section{DETAILS OF SECTION \ref{sec:concealed}}\label{app:concealed}

\thconcealed*
\begin{proof}
Fix $t$ and let $F(\cdot) = \eta_t F_t(\cdot)$. To verify the conditions of Lemma~\ref{lem:cheating regret}, recall that a Bregman divergence is non-negative, and that since $\eta_t \leq \eta_{t-1}$ we have that 
\begin{align*}
    B_{F_t}(\p, \p_1^{\phi}) = \frac{1}{\eta_t} B_F(\p, \p_1^{\phi})
    & \geq \frac{1}{\eta_{t-1}} B_F(\p, \p_1^{\phi}) = B_{F_{t-1}}(\p, \p_1^{\phi})
\end{align*}
To verify the conditions of Lemma~\ref{lem:drift} observe that $\big(\nabla F_t(\x)\big)_i$ is increasing and concave in $x(i)$ and that $\nabla^2 F_t(\x)$ positive definite, which implies that $F_t$ is Legendre---see also \citep{zimmert2020optimal}.

Thus, we may use Lemma \ref{lem:frameworkregret} to bound
\begin{align*}
    \Eb{\sumT  \inner{\q_t - \u}{\bltil_t}} 
    &\leq  \Eb{R_T(\u)}  + \Eb{\sumT (\bLhat_t - \bLhat_t^{av})^\top (\nabla^{2} R_t(\q_t))^{-1} \blhat_t}
\end{align*}
We continue by analysing the cost of implicit exploration. First, we rewrite the loss of the learner by using \citep[equation~(5)]{neu2015explore}
\begin{align*}
    \sumK q_t(i) \ltil_t(i) = \ell_t(i_t) - \epsilon_t\sumK \ltil_t(i)
\end{align*}
Since $\ltil_t(i) \leq \frac{\id[i_t = i]\ell_t(i)}{q_t(i)}$ we have that $\Eb{\ltil_t(i)} \leq 1$. This in turn implies that
\begin{align*}
    & \Eb{\sumT \inner{\q_t - \u}{\bltil_t}} \geq  \Eb{\regret_T(\u)} - \sumT \epsilon_t
\end{align*}
Similarly to~\eqref{eq:mintrick}, we use that $\big(\frac{1}{a} + \frac{1}{b}\big)^{-1} \le \min\{a,b\}$ for $a,b > 0$. Also, we use~\eqref{eq:key} and $\ell_t(i) \leq 1$ to write
\begin{align*}
    \big(\bLhat_t - \bLhat_t^{av}\big)^\top \big(\nabla^{2} R_t(\q_t)\big)^{-1}\blhat_t 
    & = \sumK \big(\Lhat_t(i) - \Lhat_t^{av}(i)\big)\left(\frac{1}{4 \eta_t q_t(i)^{3/2}} + \frac{1}{\gamma_t q_t(i)}\right)^{-1}\!\!\lhat_t(i) \\
    & \leq  \sumK \left(4 \eta_t q_t(i)^{3/2}\ltil_t(i)^2 +  \gamma_t q_t(i) \ltil_t(i)\sum_{s \in m_t(i)}\ltil_s(i)\right) \\
    & \leq  \sumK \left(4 \eta_t q_t(i)^{1/2}\ltil_t(i) +  \gamma_t \id[i_t = i]\sum_{s \in m_t(i)}\frac{\id[i_s = i]}{q_s(i)}\right)
\end{align*}
Let us now study the expectation of the first sum on the right-hand side of the last equation:
\begin{align*}
    \Eb{\sumK 4 \eta_t q_t(i)^{1/2}\ltil_t(i)} & \leq  \Eb{\sumK  4 \eta_t q_t(i)^{1/2}} 
    \leq \eta_t 4 \sqrt{K}
\end{align*}
where we used that $\max_{\p \in \triangle} \sumK \sqrt{p(i)} = \sqrt{K}$ \citep{zimmert2020optimal}. %the half-Tsallis entropy is maximized by the uniform distribution. 

Since $\epsilon_s \geq \epsilon_t$ for $s \leq t$ we have that $\sum_{s \in m_t(i)} \ltil_t(i) \leq \frac{\dstar}{\epsilon_t}$ and thus by using that $\sumT \frac{y_t}{\sqrt{\sumt y_s}} \leq 2\sqrt{\sumT y_t}$ for $y_1  > 0$ and $y_2,\ldots,y_T \ge 0$ we find that
\begin{align*}
    \Eb{\sumK \gamma_t \id[i_t = i]\sum_{s \in m_t(i)}\frac{\id[i_s = i]}{q_s(i)}} 
    & = \Eb{\frac{\sqrt{\ln(K)}\sumK \id[i_t = i]\sum_{s \in m_t(i)}\frac{\id[i_s = i]}{q_s(i)}}{\sqrt{\dstar\epsilon_t^{-1} + \sum_{s = 1}^{t-1} \sumK \id[i_s = i]\sum_{s' \in m_s(i)}\frac{\id[i_{s'} = i]}{q_{s'}(i)}}}} \\
    & \leq \Eb{\frac{\sqrt{\ln(K)}\sumK \id[i_t = i]\sum_{s \in m_t(i)}\frac{\id[i_s = i]}{q_s(i)}}{\sqrt{\sumt \sumK \id[i_s = i]\sum_{s' \in m_s(i)}\frac{\id[i_{s'} = i]}{q_{s'}(i)}}}} \\
    & \leq  \Eb{2\sqrt{\ln(K) \sumT \sumK \id[i_t = i]\sum_{s \in m_t(i)}\frac{\id[i_s = i]}{q_s(i)}}} \\
    & \leq 2\sqrt{\ln(K) \Eb{\sumT \sumK \id[i_t = i]\sum_{s \in m_t(i)}\frac{\id[i_s = i]}{q_s(i)}}}
\end{align*}
where the final inequality is due to Jensen's inequality. Following the logic used to bound~\eqref{eq:startbanditFbound}, we can evaluate the expectation: 
\begin{align*}
    \Eb{\id[i_t = i]\sum_{s \in m_t(i)}\frac{\id[i_s = i]}{q_s(i)}} = & \Eb{\sumK q_t(i) \rho_t(i)}
\end{align*}
Let us now bound $R_T(\u)$:
\begin{align*}
    R_T(\u) & \leq \frac{\sqrt{K} - \sumK u(i)}{\eta_T} + \frac{\ln(K)}{\gamma_T} \leq \frac{\sqrt{K}}{\eta_T} + \frac{\ln(K)}{\gamma_T}%= 
\end{align*}
To complete the proof we combine the above. Let
\[
    A_T = \sumT\Eb{\sumK q_t(i) \rho_t(i)}
\]
We have:
\begin{align*}
    \Eb{\sumT  \inner{\q_t - \u}{\bell_t}} & \leq  \Eb{R_T(\u)} + \sumT \epsilon_t
    + \Eb{\sumT \big(\bLhat_t - \bLhat_t^{av}\big)^\top \big(\nabla^{2} R_t(\q_t)\big)^{-1} \blhat_t}
\\&\le
    6\sqrt{KT} + 2\sqrt{\ln(K)A_T}  + 2\sqrt{T}
    + \sqrt{\ln(K)\left(\dstar\sqrt{T} + A_T\right)}
\\&\le
    6\sqrt{KT} + 2\sqrt{\ln(K)A_T}
    + 2\sqrt{T} + \sqrt{\dstar\sqrt{T} \ln(K)} + \sqrt{\ln(K)A_T}
\\&\le
    9\sqrt{KT} + 3\sqrt{\ln(K)A_T} + \half \dstar
\end{align*}
where we used that $\sqrt{ab} \leq \frac{1}{2} a + \frac{1}{2} b$ for $a, b > 0$, and that $\ln(K) \leq \sqrt{K}$.
\end{proof}

\end{document}